\documentclass[10pt,twocolumn,letterpaper]{article}

\usepackage{iccv}
\usepackage{times}
\usepackage{epsfig}
\usepackage{graphicx}
\usepackage{amsmath}
\usepackage{amssymb}
\usepackage{caption}
\usepackage{amsthm}
\usepackage{tabularx}
\usepackage[percentage]{overpic} 
\usepackage{booktabs}
\usepackage{multirow}
\usepackage{cuted}
\usepackage{makecell}
\usepackage[accsupp]{axessibility}

\usepackage{calc}

\usepackage{enumitem}
\setitemize{noitemsep,topsep=0pt,parsep=0pt,partopsep=0pt,leftmargin=*}

\setkeys{Gin}{keepaspectratio}

\makeatletter
\@namedef{ver@everyshi.sty}{}
\makeatother
\usepackage{pgfplots}
\usepgfplotslibrary{statistics}
\usepackage{pgf,tikz}
\newcommand{\circled}[1]{{\large\textcircled{\footnotesize #1}}}
\newcommand{\scircled}[1]{{\normalsize\textcircled{\scriptsize #1}}}

\usepackage{xcolor}
\definecolor{ForestGreen}{rgb}{0.13, 0.55, 0.13}
\usepackage{pifont}%

\newtheorem{lemma}{Lemma}

\newcommand{\cmark}{{\color{ForestGreen} \ding{51}}}%
\newcommand{\xmark}{{\color{red} \ding{55}}}%
\usepackage{accents}

\newcommand\blfootnote[1]{%
  \begingroup
  \renewcommand\thefootnote{}\footnote{#1}%
  \addtocounter{footnote}{-1}%
  \endgroup
}

\def\smcomb{SM-comb}

\usepackage[pagebackref=true,breaklinks=true,letterpaper=true,colorlinks,bookmarks=false]{hyperref}

\usepackage[capitalize]{cleveref}

\iccvfinalcopy %

\def\iccvPaperID{1652} %

\ificcvfinal\fi

\begin{document}

\newcommand{\vt}{\mathbf{t}}
\newcommand{\vf}{\mathbf{f}}
\newcommand{\vg}{\mathbf{g}}
\newcommand{\vh}{\mathbf{h}}
\newcommand{\vn}{\mathbf{n}}

\newcommand{\mX}{\mathbf{X}}
\newcommand{\mY}{\mathbf{Y}}
\newcommand{\mV}{\mathbf{V}}
\newcommand{\mF}{\mathbf{F}}
\newcommand{\mP}{\mathbf{P}}
\newcommand{\mR}{\mathbf{R}}
\newcommand{\mQ}{\mathbf{Q}}
\newcommand{\mS}{\mathbf{S}}
\newcommand{\mI}{\mathbf{I}}

\newcommand{\mPi}{\mathbf{\Pi}}
\newcommand{\vone}{\mathbf{1}}
\newcommand{\vzero}{\mathbf{0}}

\newcommand{\cX}{\mathcal{X}}
\newcommand{\cY}{\mathcal{Y}}
\newcommand{\cI}{\mathcal{I}}
\newcommand{\cJ}{\mathcal{J}}

\newcommand{\rmO}[1]{\mathrm{O}(#1)}
\newcommand{\rmE}[1]{\mathrm{E}(#1)}
\newcommand{\rmSO}[1]{\mathrm{S}\mathrm{O}(#1)}

\newcommand{\mFX}{\mF^{(\cX)}}
\newcommand{\mFY}{\mF^{(\cY)}}
\newcommand{\DeltaX}{\Delta^{(\cX)}}
\newcommand{\DeltaY}{\Delta^{(\cY)}}
\newcommand{\DeltadX}{{\Delta'}^{(\cX)}}
\newcommand{\DeltadY}{{\Delta'}^{(\cY)}}
\newcommand{\mXI}{\mX_\cI}
\newcommand{\mYJ}{\mY_\cJ}
\newcommand{\mPiX}{\mPi^{(\cX)}}
\newcommand{\mPiXp}{\mPi^{(\cX')}}
\newcommand{\mPiY}{\mPi^{(\cY)}}
\newcommand{\vfX}{\vf^{(\cX)}}
\newcommand{\vgX}{\vg^{(\cX)}}
\newcommand{\vnX}{\vn^{(\cX)}}
\newcommand{\vhX}{\vh^{(\cX)}}
\newcommand{\vhY}{\vh^{(\cY)}}

\newcommand{\DeltaXsub}[1]{\DeltaX_{#1}}
\newcommand{\DeltaXproj}{\DeltaXsub{\mathrm{proj}}}
\newcommand{\DeltaXstiff}{\DeltaXsub{\mathrm{stiff}}}
\newcommand{\DeltaYsub}[1]{\DeltaY_{#1}}
\newcommand{\DeltaYproj}{\DeltaYsub{\mathrm{proj}}}
\newcommand{\DeltadXsub}[1]{\DeltadX_{#1}}
\newcommand{\DeltadXproj}{\DeltadXsub{\mathrm{proj}}}
\newcommand{\DeltadYsub}[1]{\DeltadY_{#1}}
\newcommand{\DeltadYproj}{\DeltadYsub{\mathrm{proj}}}

\newcommand{\mXt}{\tilde{\mX}}
\newcommand{\mYht}{\tilde{\mYh}}
\newcommand{\mXht}{\tilde{\mXh}}
\newcommand{\mYhtJ}{\mYht_\cJ}
\newcommand{\mXhtI}{\mXht_\cI}
\newcommand{\mXtI}{\tilde{\mX}_\cI}
\newcommand{\vhXI}{\vhX_\cI}
\newcommand{\vhYJ}{\vhY_\cJ}

\newcommand{\mXh}{\hat{\mX}}
\newcommand{\mYh}{\hat{\mY}}
\newcommand{\mXhI}{\mXh_\cI}
\newcommand{\mYhJ}{\mYh_\cJ}

\newcommand{\fcX}{f_\cX}
\newcommand{\fcY}{f_\cY}
\newcommand{\ncX}{n_\cX}
\newcommand{\ncY}{n_\cY}
\newcommand{\mcX}{m_\cX}
\newcommand{\mcY}{m_\cY}
\newcommand{\dcX}{d_\cX}
\newcommand{\dcY}{d_\cY}
\newcommand{\scX}{s_\cX}
\newcommand{\scY}{s_\cY}

\newcommand{\absmX}{|\mX|}
\newcommand{\absmY}{|\mY|}

\newcommand{\bbR}{\mathbb{R}}
\newcommand{\bbN}{\mathbb{N}}

\newcommand{\rangei}[1]{\{1,\dots,#1\}}

\newcommand{\pz}{\phantom{0}}

\newcommand{\dummyfig}[1]{
  \centering
  \fbox{
    \begin{minipage}[c][0.33\textheight][c]{0.5\textwidth}
      \centering{#1}
    \end{minipage}
  }
}

\newcommand{\cErec}{\mathcal{E}_\mathrm{rec}}
\newcommand{\cEdef}{\mathcal{E}_\mathrm{def}}
\newcommand{\cEori}{\mathcal{E}_\mathrm{ori}}

\newcommand{\lambdadef}{\lambda_\mathrm{def}}
\newcommand{\lambdaori}{\lambda_\mathrm{ori}}
\newcommand{\lambdarec}{\lambda_\mathrm{rec}}
\newcommand{\expnumber}[2]{{#1}\mathrm{e}{#2}}

\def\pathOurs{figs/rend_ours/}
\def\pathChen{figs/rend_chen/}
\def\pathMina{figs/rend_mina/}
\def\pathPmsdp{figs/rend_pmsdp/}
\def\pathsmcomb{figs/rend_smcomb/}
\def\srcEnd{_M}
\def\trgtEnd{_N}

\title{$\boldsymbol{\Sigma}$IGMA: Scale-Invariant Global Sparse Shape Matching}%
\thispagestyle{plain}
\pagestyle{plain}

\author{ Maolin Gao$^{1,4}$ ~~~~ Paul Roetzer$^3$ ~~~~ Marvin Eisenberger$^{1,4}$ \\ Zorah L{\"a}hner$^{2}$ ~~~~ Michael Moeller$^2$ ~~~~ Daniel Cremers$^{1,4,5}$ ~~~~ Florian Bernard$^3$\\ \\
$^1$ Technical University of Munich \quad $^2$ University of Siegen \quad $^3$ University of Bonn \\ $^4$ Munich Center for Machine Learning \quad $^5$ University of Oxford
}

\maketitle%
\ificcvfinal\thispagestyle{empty}\fi%
\maketitle%
\begin{strip}%
  \centerline{%
  \footnotesize%
  \begin{tabular}{cccc}%
    \setlength{\tabcolsep}{5pt}
    \begin{overpic}[width=3cm]{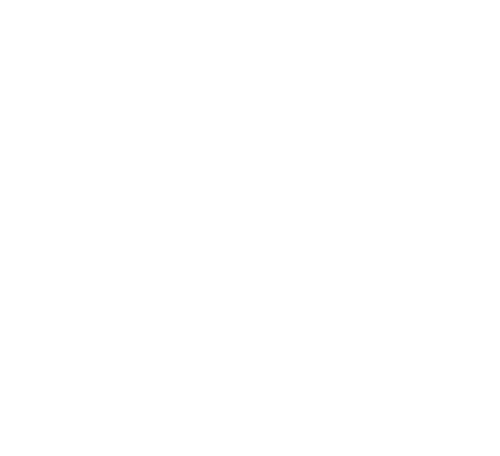}
        \put(0, 0){\includegraphics[height=4.2cm]{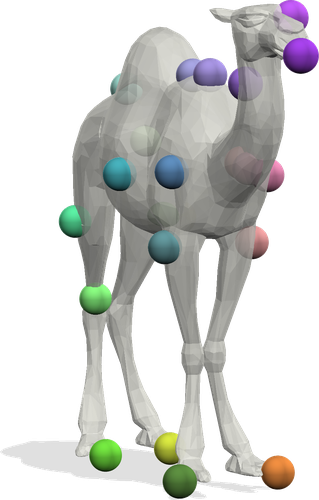}}
    \end{overpic}
    \hspace{-0.2cm}
    \includegraphics[height=1.4cm]{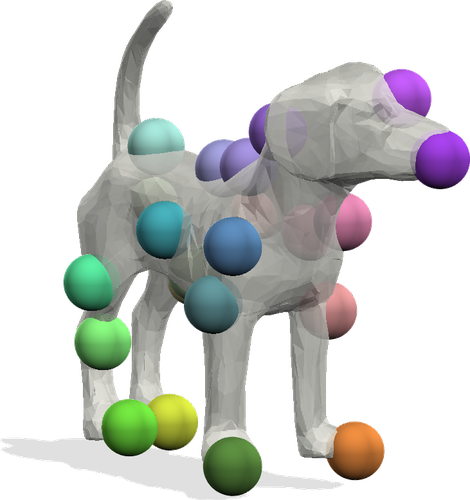}
    &
    \hspace{-0.25cm}
    \includegraphics[height=2.0cm]{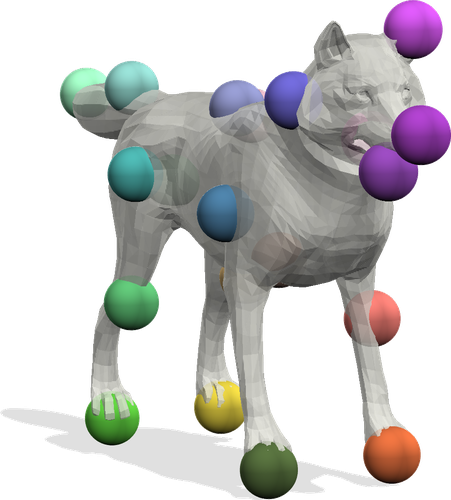}
    \hspace{-0.25cm}
    \begin{overpic}[width=2.5cm]{figs/empty.png}
        \put(0, 0){\includegraphics[height=3.8cm]{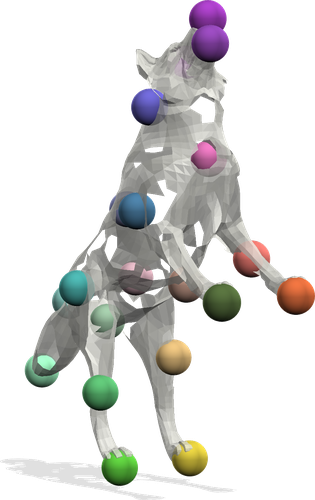}}
    \end{overpic}
    &
    \hspace{-0.75cm}
    \newcommand{\gapLineWidth}{3pt}
\definecolor{cPLOT0}{RGB}{28,213,227}
\definecolor{cPLOT1}{RGB}{80,150,80}
\definecolor{cPLOT2}{RGB}{90,130,213}
\definecolor{cPLOT3}{RGB}{247,179,43}
\definecolor{cPLOT5}{RGB}{242,64,0}

\pgfplotsset{%
    label style = {font=\large},
    tick label style = {font=\large},
    title style =  {font=\Large},
    legend style={  fill= gray!10,
                    fill opacity=0.6, 
                    font=\large,
                    draw=gray!20, %
                    text opacity=1}
}
\begin{tikzpicture}[scale=0.5, transform shape]
	\begin{axis}[
		width=\columnwidth,
		height=0.75\columnwidth,
		grid=major,
		title=Optimality vs. Runtime,
		legend style={
			at={(0.97,0.03)},
			anchor=south east,
			legend columns=1},
		legend cell align={left},
		ylabel={{\large$\%$ Pairs w/ Closed Gap}},
        xlabel={Runtime (s)},
		xmin=0,
        xmax=3650,
        ylabel near ticks,
        xtick={0, 1000, 2000, 3000, 4000},
		ymin=-5,
        ymax=85,
        ytick={0, 20, 40, 60, 80, 100},
	]
    \addplot [color=cPLOT0, smooth, line width=\gapLineWidth]
    table[row sep=crcr]{%
0 0\\
124.1379 0\\
248.2759 1.4085\\
372.4138 4.2254\\
496.5517 4.2254\\
620.6897 4.2254\\
744.8276 4.2254\\
868.9655 4.2254\\
993.1034 4.2254\\
1117.2414 4.2254\\
1241.3793 4.2254\\
1365.5172 4.2254\\
1489.6552 4.2254\\
1613.7931 4.2254\\
1737.931 4.2254\\
1862.069 4.2254\\
1986.2069 4.2254\\
2110.3448 4.2254\\
2234.4828 4.2254\\
2358.6207 4.2254\\
2482.7586 4.2254\\
2606.8966 4.2254\\
2731.0345 4.2254\\
2855.1724 4.2254\\
2979.3103 4.2254\\
3103.4483 4.2254\\
3227.5862 4.2254\\
3351.7241 4.2254\\
3475.8621 4.2254\\
3600 4.2254\\
    };
    \addlegendentry{\textcolor{black}{\smcomb}~\cite{roetzer2022scalable}}

    \addplot [color=cPLOT2, smooth, line width=\gapLineWidth]
    table[row sep=crcr]{%
0 0\\
124.1379 0\\
248.2759 0\\
372.4138 0\\
496.5517 0\\
620.6897 0\\
744.8276 0\\
868.9655 0\\
993.1034 0\\
1117.2414 0\\
1241.3793 0\\
1365.5172 0\\
1489.6552 0\\
1613.7931 0\\
1737.931 0\\
1862.069 0\\
1986.2069 0\\
2110.3448 0\\
2234.4828 0\\
2358.6207 0\\
2482.7586 0\\
2606.8966 0\\
2731.0345 0\\
2855.1724 0\\
2979.3103 0\\
3103.4483 0\\
3227.5862 0\\
3351.7241 0\\
3475.8621 0\\
3600 0\\
    };
    \addlegendentry{\textcolor{black}{PMSDP}~\cite{maron2016point}}

    \addplot [color=cPLOT3, smooth, line width=\gapLineWidth]
    table[row sep=crcr]{%
0 0\\
124.1379 0\\
248.2759 0\\
372.4138 1.4085\\
496.5517 2.8169\\
620.6897 2.8169\\
744.8276 2.8169\\
868.9655 2.8169\\
993.1034 2.8169\\
1117.2414 4.2254\\
1241.3793 7.0423\\
1365.5172 7.0423\\
1489.6552 7.0423\\
1613.7931 7.0423\\
1737.931 8.4507\\
1862.069 8.4507\\
1986.2069 9.8592\\
2110.3448 9.8592\\
2234.4828 9.8592\\
2358.6207 9.8592\\
2482.7586 9.8592\\
2606.8966 9.8592\\
2731.0345 9.8592\\
2855.1724 11.2676\\
2979.3103 11.2676\\
3103.4483 11.2676\\
3227.5862 11.2676\\
3351.7241 11.2676\\
3475.8621 11.2676\\
3600 11.2676\\
    };
    \addlegendentry{\textcolor{black}{Mina}~\cite{bernard2020mina}}

    \addplot [color=cPLOT5, smooth, line width=\gapLineWidth]
    table[row sep=crcr]{%
0 0\\
124.1379 18.3099\\
248.2759 28.169\\
372.4138 29.5775\\
496.5517 36.6197\\
620.6897 40.8451\\
744.8276 46.4789\\
868.9655 54.9296\\
993.1034 54.9296\\
1117.2414 56.338\\
1241.3793 57.7465\\
1365.5172 59.1549\\
1489.6552 60.5634\\
1613.7931 61.9718\\
1737.931 63.3803\\
1862.069 63.3803\\
1986.2069 67.6056\\
2110.3448 69.0141\\
2234.4828 69.0141\\
2358.6207 69.0141\\
2482.7586 70.4225\\
2606.8966 71.831\\
2731.0345 71.831\\
2855.1724 74.6479\\
2979.3103 76.0563\\
3103.4483 76.0563\\
3227.5862 77.4648\\
3351.7241 77.4648\\
3475.8621 77.4648\\
3600 78.8732\\
    };
    \addlegendentry{\textcolor{black}{Ours}}
        
	\end{axis}
\end{tikzpicture}
    &
    \hspace{-0.8cm}
    \newcommand{\gapLineWidth}{3pt}
\definecolor{cPLOT0}{RGB}{28,213,227}
\definecolor{cPLOT1}{RGB}{80,150,80}
\definecolor{cPLOT2}{RGB}{90,130,213}
\definecolor{cPLOT3}{RGB}{247,179,43}
\definecolor{cPLOT5}{RGB}{242,64,0}

\pgfplotsset{%
    label style = {font=\large},
    tick label style = {font=\large},
    title style =  {font=\Large},
    legend style={  fill= gray!10,
                    fill opacity=0.6, 
                    font=\large,
                    draw=gray!20, %
                    text opacity=1}
}
\begin{tikzpicture}[scale=0.5, transform shape]
	\begin{axis}[
		width=\columnwidth,
		height=0.75\columnwidth,
		grid=major,
		title=Runtime vs. Shape Resolution,
		legend style={
			at={(0.97,0.86)},
			anchor=north east,
			legend columns=1},
		legend cell align={left},
		ylabel={{\large Runtime (s$\cdot10^3$)}},
        xlabel={\# Faces $\cdot10^3$},
        xmin=0,
        xmax=9000,
        ylabel near ticks,
        xtick={2000, 4000, 6000, 8000, 10000},
        xticklabels={ $2$, $4$, $6$, $8$, $10$},
        xtick scale label code/.code={},
        ymin=-100,
        ymax=4000,
        ytick={0, 1000, 2000, 3000, 4000},
        yticklabels={ $0$, $1$, $2$, $3$, $4$},
        extra y ticks=3600,
        extra y tick labels={},
        extra y tick style={grid=major, grid style={dashed,black!70,line width=2pt}},
	]
\node[anchor=east] at (axis cs: 8800,3800) {{\large Max. Time Budget}};
    \addplot [color=cPLOT0, smooth, line width=\gapLineWidth]
    table[row sep=crcr]{%
0 0\\
200 230.47\\
400 821.73\\
582 3590.9\\
    };
    \addlegendentry{\textcolor{black}{\smcomb}~\cite{roetzer2022scalable}}

        \addplot [color=cPLOT0, smooth, dotted, line width=\gapLineWidth, forget plot]
    table[row sep=crcr]{%
582 3590.9\\
600 3980.3\\
    };
    \addplot [color=cPLOT2, smooth, line width=\gapLineWidth]
    table[row sep=crcr]{%
0 0\\
600 8.6774\\
1200 8.9443\\
1800 9.1023\\
2400 8.8943\\
3000 9.0454\\
3600 9.0661\\
4200 8.792\\
4800 8.619\\
5400 8.8934\\
6000 8.9223\\
6600 8.8437\\
7200 8.7901\\
7800 8.878\\
8400 8.6731\\
9000 9.3577\\
    };
    \addlegendentry{\textcolor{black}{PMSDP}~\cite{maron2016point}}
    \addplot [color=cPLOT3, smooth, line width=\gapLineWidth]
    table[row sep=crcr]{%
0 0\\
600 1125.2424\\
1200 1671.7544\\
1800 1630.4576\\
2400 3600.4906\\
    };
    \addlegendentry{\textcolor{black}{MINA}~\cite{bernard2020mina}}
\addplot [color=cPLOT3, dotted, smooth, line width=\gapLineWidth, forget plot]
    table[row sep=crcr]{%
2400 3600.4906\\
3000 7581.8534\\
3600 13574.546\\
    };
    \addplot [color=cPLOT5, smooth, line width=\gapLineWidth]
    table[row sep=crcr]{%
0 0\\
600 24.4268\\
1200 26.0234\\
1800 26.0311\\
2400 37.5955\\
3000 49.7587\\
3600 53.0488\\
4200 77.0996\\
4800 119.2951\\
5400 173.051\\
6000 205.4263\\
6600 318.8348\\
7200 273.4554\\
7800 340.2952\\
8400 399.3415\\
9000 599.1073\\
    };
    \addlegendentry{\textcolor{black}{Ours}}
        
	\end{axis}
\end{tikzpicture}\\
    (a) Non-rigidity \& scale difference &
    (b) Partiality \& scale difference &
    (c) Global optimality (TOSCA)&
    (d) Scalability
  \end{tabular}
  }
  \vspace{-3mm}

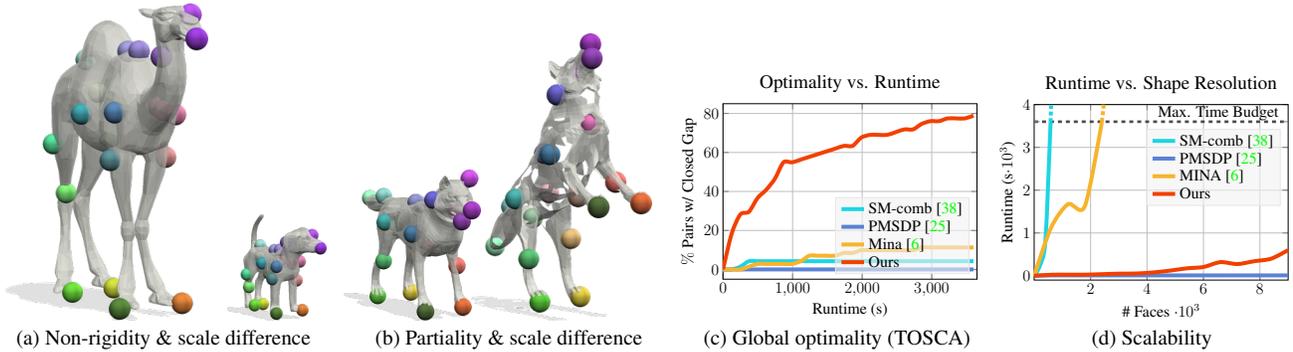
\captionof{figure}{We propose a sparse non-rigid shape matching approach that is provably invariant to rigid transformations and global scaling, can (often) be solved to optimality, and scales linearly with mesh resolution. (a, b) Our matchings for non-rigid shapes with drastically different scale and  partiality. 
(c, d) Our method is the only one that is able to both solve the majority of  pairs to global optimality within a time budget of $1$h, as well as scaling up to high mesh resolutions. }
\label{fig:teaser}
\end{strip}

\begin{abstract}
We propose a novel mixed-integer programming (MIP) formulation for generating precise sparse correspondences for highly non-rigid shapes. 
To this end, we introduce a projected Laplace-Beltrami operator (PLBO) which combines intrinsic and extrinsic geometric information to measure the deformation quality induced by predicted correspondences.
We integrate the PLBO, together with an orientation-aware regulariser, into a novel MIP formulation that can be solved to global optimality for many practical problems.
In contrast to previous methods, our approach is provably invariant to rigid transformations and global scaling, initialisation-free, has optimality guarantees, and scales to high resolution meshes with (empirically observed) linear time.
We show state-of-the-art results for sparse non-rigid matching on several challenging 3D datasets, including data with inconsistent meshing, as well as applications in mesh-to-point-cloud matching.
\end{abstract}
\vspace{-2.5mm}
\section{Introduction} \label{sec:introduction}
Finding correspondences or matchings between parts of 3D shapes is a well-studied problem that lies at the heart of many tasks in computer vision, computer graphics and beyond. 
Example tasks related to such 3D shape matching problems include generative shape modelling \cite{wu2019sagnet}, 3D shape analysis \cite{ren2020jaws}, motion capture \cite{taylor2012vitruvian}, or model-based image segmentation~\cite{bernard2017combinatorial}, which are relevant for applications in autonomous driving, robotics, and biomedicine. 

\begin{table*}[h]
\resizebox{\linewidth}{!}{%
\setlength{\tabcolsep}{12pt}
\begin{tabular}{@{}lclccccl@{}}
\toprule
\textbf{Method}                        & \textbf{Init.-free} & \textbf{Solver} & \textbf{\% of optimal pairs$^\dagger$}                                      & \textbf{Scale Inv.} & \textbf{Rigid  Inv.} & \textbf{Data}      \\ \midrule
PMSDP \cite{maron2016point}            & \cmark              & convex solver               & $5.0\%$                       &  \cmark             &  \cmark              & point cloud        \\
MINA \cite{bernard2020mina}            & \cmark              & MIP solver          & $11.5\%$ & \xmark              &  (\cmark)$^\star$              & mesh \& point cloud \\
\smcomb~\cite{roetzer2022scalable}     & \xmark              & custom solver                & $23.5\%^\ddagger$      &  \xmark             &  \cmark              & mesh               \\
\textbf{Ours}             & \cmark                           & MIP solver           & $73.0\%$       & \cmark              &  \cmark              & mesh \& point cloud      \\ \bottomrule          
\end{tabular}
}
\caption{\textbf{Overview of axiomatic shape matching approaches} that have a global optimisation flavour. 
Our approach is the only one that at the same time is initalisation-free, scale- and rigid motion-invariant, and works on both meshes and point clouds. Moreover, ours achieves the best proportion of globally optimal solutions across all experiments. 
($^\dagger$summarised across all datasets. $^\ddagger$w/o runtime budget. $^\star$MINA explicitly optimises for a global rotation using an $\rmSO{3}$  discretisation.)
}
\label{tab:properties}
\end{table*}

While 3D shape matching is traditionally addressed in terms of an optimisation problem formulation, in recent years learning-based approaches became more popular. 
Specifically, with the advent of geometric deep learning we have witnessed a dramatic improvement in the performance of data-driven methods for 3D shape matching, even for difficult settings such as unsupervised learning for partial shapes~\cite{eisenberger2020deep, attaiki2021dpfm}.
Yet, such learning-based approaches lack theoretical guarantees, both regarding the (global) optimality and often also regarding  structural properties of obtained solutions. 
While  optimality guarantees are crucial in safety-critical domains (e.g. autonomous driving, or computer-aided surgery), desirable structural properties may be related to geometric consistency, such as smoothness or the continuity of matchings~\cite{windheuser2011geometrically}. 
From a technical point of view it is typically straightforward  to impose respective structural properties within an optimisation-based framework. 
However, resulting formulations are high-dimensional and non-convex, so that finding `good' solutions is a major challenge -- many resulting problems lead to large-scale integer linear programming formulations, e.g.~when imposing discrete diffeomorphism constraints~\cite{windheuser2011geometrically}, or when considering the NP-hard quadratic assignment problem~\cite{Lawler:1963wn,Pardalos:1993uo,rendl1994quadratic,rodola2012game,kezurer2015,vestner2017efficient}.

Yet, 3D shape matching typically involves profound structural characteristics (e.g.~related to geometric constraints), so that, despite the complexity and high dimensional nature of respective optimisation problems, several \linebreak formalisms have been proposed that can be efficiently solved in practice.
Often they are built on a low dimensional matching representation, e.g.~in the spectral domain~\cite{ovsjanikov2012functional}, or in terms of a sparse set of discrete control points that give rise to a dense correspondence~\cite{bernard2020mina}. 
In this work we build upon the latter and propose a novel formalism that has a range of desirable properties, including scale invariance, rigid motion invariance and a substantially better scalability (compared to~\cite{bernard2020mina}), which directly leads to a major improvement of the matching quality.
In Tab.~\ref{tab:properties} we provide an overview of the properties of our method and its direct competitors.

\textbf{Contributions.}
We propose a novel mixed-integer programming formulation for \textbf{S}cale-\textbf{I}nvariant \textbf{G}lobal sparse shape \textbf{Ma}tching (\textbf{SIGMA}) based on a projected Laplace-Beltrami operator, which combines intrinsic and extrinsic geometric information and is applicable to a variety of challenging non-rigid shape matching problems.
In summary, our main contributions are:
\begin{itemize}
    \setlength\itemsep{0em}
    \item A novel initialisation-free mixed-integer programming formulation for sparse shape matching, which can be solved to global optimality for many practical instances and (empirically) scales linearly with the mesh resolution.
    \item Our method is provably invariant to rigid transformations and global scaling, thus eliminating the extrinsic alignment required by many shape matching pipelines.
    \item We propose the use of the \emph{projected} Laplace-Beltrami operator for geometry reconstruction, which combines intrinsic and extrinsic geometric information, while still being invariant under  actions of the Euclidean group $\rmE{3}$.
    \item We obtain state-of-the-art results on multiple challenging non-rigid shape matching datasets.
\end{itemize}

\section{Related Work} \label{sec:relatedwork}

Shape matching is a widely studied topic and providing an in-depth survey would be beyond the scope of this paper.
We refer the reader to \cite{sahillioglu2020survey} for an overview of non-rigid matching and only mention directly related literature here.

\textbf{Quadratic Assignment Problem.}
The shape matching problem can be formulated as a quadratic assignment problem (QAP) \cite{Loiola:2ua4FrR7}. 
The respective QAP aims to find the optimal permutation between (vertices of) two shapes preserving a predefined pairwise property. 
For non-rigid shape matching this property is often chosen as the geodesic distance between pairs of vertices. 
However, the QAP is NP-hard~\cite{rendl1994quadratic},
so that most methods rely on convex relaxations \cite{Schellewald:2005up,Dym:2017ue,bernard:2018,kushinsky2019sinkhorn}, or  heuristics \cite{le2017alternating,vestner2017efficient}. 
In consequence, these methods cannot guarantee to find the global optimum of the original non-convex problem  and can thus result in arbitrarily bad solutions.
Instead of considering a QAP formulation that seeks to find a permutation between all shape vertices, our approach uses a low-dimensional, discrete matching representation that is combined with dense geometry reconstruction, thereby requiring only the computation of a permutation between a small number of keypoints.

\textbf{Matching with Optimality Guarantees.}
Even though computing global optima is intractable for general QAP formulations, some existing works that consider more specialised shape matching models can be solved efficiently.
This is the case for contour-to-contour \cite{schmidt2009planar}, contour-to-image \cite{felzenszwalb2005representation,schoenemann2009combinatorial}, and contour-to-surface \cite{lahner2016efficient, roetzer2023conjugate} matchings.
An analogous approach has been proposed for surface-to-surface matching in \cite{windheuser2011geometrically}. However, it does not scale to high resolutions and global optimality cannot be guaranteed.
In \cite{roetzer2022scalable} an efficient decomposition of \cite{windheuser2011geometrically} has been introduced which scales to higher resolutions but still cannot guarantee optimality.
The approach of \cite{bernard2020mina} uses a convex mixed-integer alignment model to find correspondences between two surfaces. While its  worst-case  time complexity is exponential, it converges to a global optimum in reasonable time in many practical settings. 

Another direction was taken by \cite{ovsjanikov2012functional} in which  point-wise correspondences are transformed to functional correspondences which can then be reduced to a low-dim, continuous problem via the Laplace-Beltrami eigenfunctions.
The optimisation for this so-called functional map can be done to global optimality depending on the energy in the spectral domain, but it is hard to guarantee properties like bijectivity for the underlying point-wise correspondence, despite several recent improvement attempts \cite{Pai2021CVPR, Ren2020MapTree}.

\textbf{Sparse Matching.}
Sparse matching methods only compute correspondences for a small subset of vertices of the input.
This allows to impose constraints that are otherwise intractable for  full resolution shapes, and therefore leads to few but often high quality matches. 
Many dense correspondence methods can only find local optima \cite{vestner2017efficient,ezuz2019elastic}, and hence rely on a meaningful initialisation, which can f.e. be given by a robust sparse set of initial correspondences \cite{sahillioglu2018genetic, yang2015sparse,panine2022landmark}.
The sparsity of the solution in \cite{rodola2012game} is achieved by utilising an L1-norm formulation, directly promoting sparsity.
In \cite{maron2016point} the problem of non-rigid matching is posed as high-dimensional procrustes matching problem which is convexified so that the relaxed problem can be solved to global optimality.
The MINA approach of \cite{bernard2020mina} defines a deformation model based on sparse keypoints, and eventually optimises via  mixed-integer programming. 
While conceptually \cite{bernard2020mina} resembles our approach,
it explicitly solves for a rotation matrix. In contrast, ours does not require this and is thus significantly more efficient, so that we reduce optimality gaps substantially faster, see Fig.~\ref{fig:teaser}c.

\textbf{Scale Invariant Shape Analysis.} Shapes with different global scaling are ubiquitous, and techniques such as shape normalisation and scale-invariant feature descriptors have been developed \cite{aubry2011wks, BronsteinK10}.
A scale-invariant metric incorporating Gaussian curvature as the adaptive normalisation factor has been proposed in \cite{aflalo2013saleinvariant}. Based on that, an intrinsic scale-invariant Laplacian-Beltrami Operator (LBO) can be constructed. In contrast, our projected LBO is also invariant to rigid motions and blended with extrinsic information, which is found useful for shape reconstruction.

\section{Global Sparse Shape Matching} \label{sec:method}

We address the task of sparse non-rigid shape matching, which we summarise in the following. 
For each problem instance, we consider a pair of surfaces $\cX$ and $\cY$ which are given as triangle mesh discretisations of Riemannian 2-manifolds embedded in 3D space. 
These are defined as the triplets $\cX=\bigl(\mX,\mFX,\cI\bigr)$ and $\cY=\bigl(\mY,\mFY,\cJ\bigr)$, where the first element denotes the vertices, the second element the faces, and the third element the indices of keypoints (\ie~a subset of vertex indices).
Tab.~\ref{table:notation} summarises our notation. 

Our aim is to determine optimal correspondences between the keypoint vertices $\mXI$ and $\mYJ$ (of $\cX$ and $\cY$, respectively), which we represent as a permutation matrix $\mP\in\{0,1\}^{n \times n}$. In order to find this permutation, we consider a regularisation that utilises synergies between shape matching and shape reconstruction.
To this end, we utilise \emph{rigid motion-invariant} geometric information of shape $\cX$, and then use keypoints of shape $\cY$ in order to reconstruct shape $\cX$ in the pose of shape $\cY$ (and vice-versa with swapped roles of $\cX$ and $\cY$), see Fig.~\ref{fig:intuition}.

For the rigid motion-invariant geometry encapsulation  we introduce the projected Laplace-Beltrami operator in Sec.~\ref{subsec:extrinsic-deformation} and our sparse matching formalism in Sec.~\ref{subsec:sparse-matching}.

\begin{table}[t!]
\small\centering
	\begin{tabularx}{\columnwidth}{lp{5.6cm}}
 \toprule
        \textbf{Symbol} & \textbf{Description} \\
		\toprule
		  $\cX=\bigl(\mX,\mFX,\cI\bigr)$ & shape $\cX$\\
        $\absmX, |\mFX|$   & number of vertices and faces in $\cX$ \\
            $\mX \in \bbR^{\absmX\times 3} $ & vertices of shape $\cX$\\
            $\mFX \in \bbN^{|\mFX| \times 3} $ & faces of shape $\cX$\\
            $\cI \in \bbN^{n}$ & keypoint indices on shape $\cX$\\
            $ \mX_{\cI} \in \bbR^{n \times 3} $ & keypoint coordinates on shape $\cX$\\
            $\mXt \in \bbR^{\absmX\times 4}$ & vertices of shape $\cX$ in homogeneous coordinates (appended with $\vone$) \\
            $\mXh\in \bbR^{\absmX\times 3} $ & reconstruction of $\mX$ with connectivity of $\cX$ and pose of $\cY$ (same vertex ordering  as $\mX$) \\
            $\mPiX \in \bbR^{\absmX \times \absmX}$ & projection onto the null space of $\mXt$ \\
            $\vhX \in \bbR^{\absmX} $ & orientation-aware features of $\cX$\\
            $\vhXI \in \bbR^{n} $ & orientation-aware features of $\cX$'s keypoints\\
            $\cY=\bigl(\mY,\mFY,\cJ\bigr)$ & shape $\cY$\\
            $\vdots$ & (analogous as above)\\
            $\mP\in\{0,1\}^{n \times n}$ & permutation of keypoints\\
		\bottomrule
	\end{tabularx}
	\caption{Summary of our notation.}
	\label{table:notation}
\end{table}

\begin{figure}
    \centering%
    \includegraphics[width=\columnwidth]{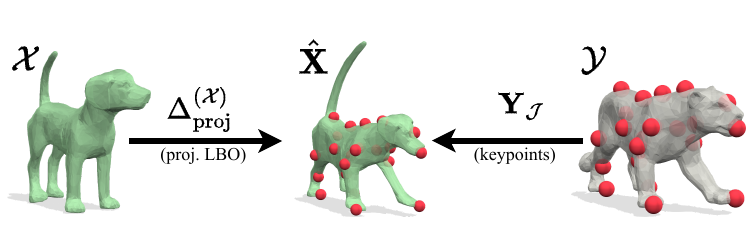}
    \caption{We \textbf{utilise synergies between shape matching  and shape reconstruction}. To this end, rigid motion-invariant geometric information of shape $\cX$, encoded in the projected LBO $\DeltaXproj$, and the keypoint coordinates $\mYJ$ of shape $\cY$, are combined to reconstruct shape $\cX$ in the pose of shape $\cY$. Our motivation is that high quality correspondences will lead to better reconstruction and vice versa.}
    \label{fig:intuition}
\end{figure}

\subsection{Projected Laplace-Beltrami Operator}\label{subsec:extrinsic-deformation}

To encapsulate rigid motion-invariant geometric information,
we propose a variant of the Laplace-Beltrami operator that we call \emph{Projected Laplace-Beltrami operator} (PLBO). For $\DeltaXstiff$ being the stiffness matrix component of the Laplacian~\cite{pinkall1993cotan}, we define the PLBO as
\begin{align}\label{eq:proj-lbo}
    \DeltaXproj:=\bigl(\mPiX\bigr)^\top\DeltaXstiff\mPiX,
\end{align}
where $\mPiX$ denotes the projection matrix
\begin{align}\label{eq:proj-matrix}
    \mPiX:=\mathbf{I}-\mXt(\mXt^\top\mXt)^{-1}\mXt^\top,~~\text{with}~\mXt:=\begin{pmatrix}\mX&\vone\end{pmatrix},
\end{align}
and $\vone\in\bbR^{\absmX}$ is a vector of all ones.~\nocite{nasikun2018fastspectrum}
To obtain an operator that is scale-invariant, the definition of Eqn.~\eqref{eq:proj-lbo} is area-normalised, \ie based only on the stiffness matrix component $\DeltaXstiff$ of the Laplacian.
Due to the use of the coordinate function, the PLBO is not purely intrinsic but also incorporates extrinsic information.
We found this mix of intrinsic and extrinsic information beneficial for accurate correspondence computation, and even though the PLBO uses the extrinsic coordinate function, it is still invariant under transformations in the Euclidean group $\rmE{3}$:
\begin{lemma}\label{lemma:invariance-rigid-body}
Let $\Delta(\mX):=\DeltaXproj\in\bbR^{\absmX\times \absmX}$ be the projected Laplace-Beltrami operator for the vertices $\mX$, defined in~Eqn.~\eqref{eq:proj-lbo}. For any rigid body transformation 
\begin{equation}
    \begin{pmatrix}
    \mR & \vt \\
    \vzero & 1
    \end{pmatrix}\in\rmE{3},~~\text{with}~~\mR\in \rmO{3},\vt\in\bbR^3,
\end{equation}
it holds that $\Delta(\mX)=\Delta(\mX\mR^\top+\vone\vt^\top)$.
\end{lemma}
See the supplementary material for the proof.

While there are several popular deformation models in the literature~\cite{botsch2006primo,grinspun2003discrete,solomon2011killing,sorkine2007rigid}, we find that existing formulations often depend on the extrinsic pose of the input surfaces. 
Therefore, such methods require rigidly aligned poses since applying a rotational offset $\mX\mR^\top$ can alter the results for any $\mR\in\rmO{3}$. 
Compared to the standard LBO operator $\DeltaX\in\bbR^{\absmX\times\absmX}$, the PLBO $\DeltaXproj$ projects the original coordinate function $\mX$ onto its null space, so that when used as a regulariser (as we introduce in Eqn.~\eqref{eq:deformTerm}) only components outside the null space are penalised.
In practice, this reduces oversmoothing and leads to more accurate reconstructions, see Fig.~\ref{fig:lbovsproj} for an illustration.

\def\wLBO{2cm}
\def\hLBO{2cm}
\begin{figure}[h]
    \centering
    \begin{tabular}{cccccc}
        \rotatebox{90}{\hspace{0.5cm}LBO} &
        \includegraphics[height=\hLBO, width=\wLBO]{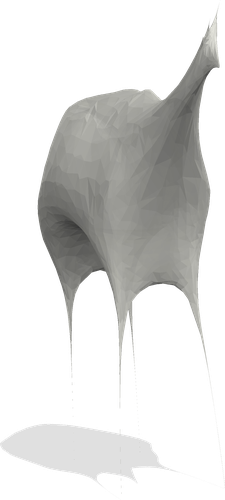} & 
        \includegraphics[height=\hLBO, width=\wLBO]{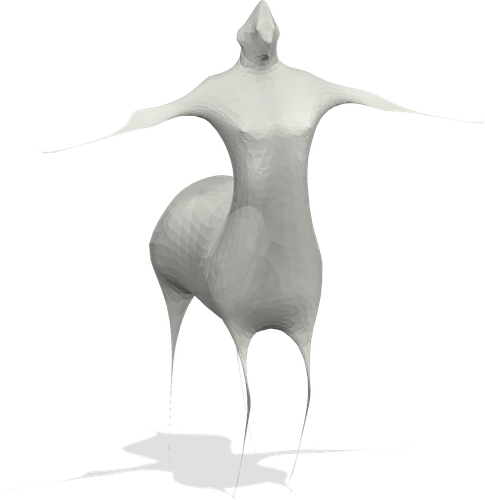} & 
        \includegraphics[height=\hLBO, width=\wLBO]{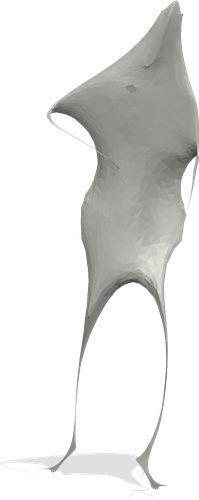} & 
        \includegraphics[height=\hLBO, width=\wLBO]{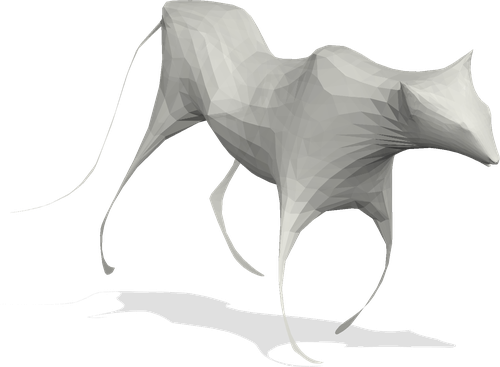} \\
        \rotatebox{90}{\hspace{0.2cm}proj. LBO} &
        \includegraphics[height=\hLBO, width=\wLBO]{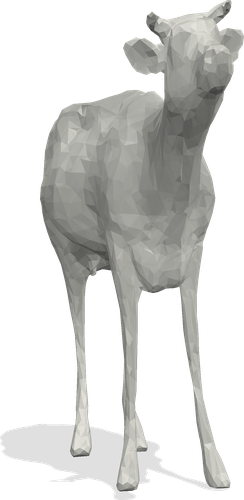} &
        \includegraphics[height=\hLBO, width=\wLBO]{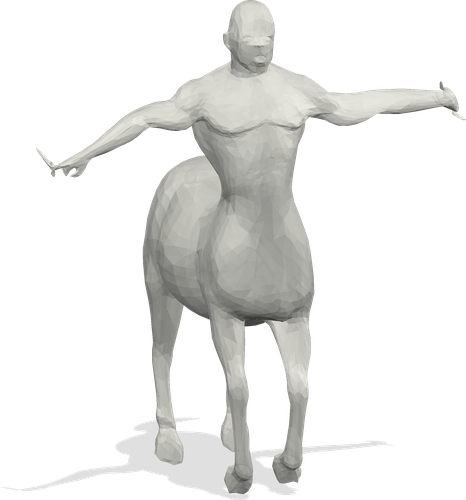} &
        \includegraphics[height=\hLBO, width=\wLBO]{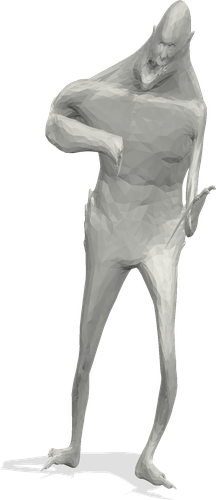} & 
        \includegraphics[height=\hLBO, width=\wLBO]{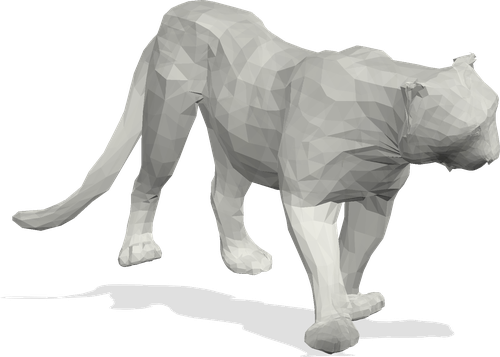} \\
    \end{tabular}
     \caption{Comparison of \textbf{reconstructed shapes} using the standard \textbf{LBO} and our proposed \textbf{projected LBO}. From the definition in Eqn.~\eqref{eq:proj-lbo} we can show that $\DeltaXproj\mX=0$, \ie this yields a deformation prior that incurs no cost on exact reconstructions of the initial geometry $\mX$. Thus, the projected LBO is able to better preserve local geometry and leads to more realistic reconstructions. See the supplementary material for details.
     }
    \label{fig:lbovsproj}
\end{figure}

\subsection{Sparse Matching}\label{subsec:sparse-matching}

Our final optimisation problem comprises a reconstruction term $\mathcal{E}_\mathrm{rec}$, a deformation term $\mathcal{E}_\mathrm{def}$ and an orientation-aware term $\mathcal{E}_\mathrm{ori}$, and reads
\begin{align}\label{eq:optimization-problem}
    \min_{\mP, \mXh, \mYh}&~\mathcal{E}_\mathrm{rec}(\mP, \mXh, \mYh) + \lambda_\mathrm{def} \mathcal{E}_\mathrm{def}(\mXh, \mYh) + \lambda_\mathrm{ori} \mathcal{E}_\mathrm{ori}(\mP) \nonumber\\
    \text{s.t.}~~&~~\mP\in\{0,1\}^{n\times n},~\mP^\top\vone_n=\vone_n,~\mP\vone_n=\vone_n.
\end{align}
Our formulation is provably invariant under global scaling and rigid-body motions of the inputs:

\begin{lemma}\label{lemma:invariance-scale-rigid}
    Let $\bigl(\mP,\mXh,\mYh\bigr)$ be a global optimiser of Eqn.~\eqref{eq:optimization-problem}. 
    \begin{enumerate}
        \item[(a)] Let $\cX':=\bigl(s\mX,\mFX,\cI\bigr)$ be a rescaled input shape $\cX$, where a scalar factor $s>0$ is applied to the vertex coordinates.
        Then $\bigl(\mP',\mXh',\mYh'\bigr):=\bigl(\mP,\mXh,s\mYh\bigr)$ is a global optimiser of Eqn.~\eqref{eq:optimization-problem} between $\cX'$ and $\cY$.
        \item[(b)] Let $\cX'':=\bigl(\mX \mR^\top + \vone\vt^\top ,\mFX,\cI\bigr)$ be a rigidly transformed version of $\cX$ with $\mR\in \rmSO{3},\vt\in\bbR^3$.
    Then $\bigl(\mP'',\mXh'',\mYh''\bigr):=\bigl(\mP,\mXh,\mYh\mR^\top+\vone\vt^\top\bigr)$ is a global optimiser of Eqn.~\eqref{eq:optimization-problem} between $\cX''$ and $\cY$.
    \end{enumerate}
\end{lemma}
Lemma~\ref{lemma:invariance-scale-rigid} holds analogously for rescaling and rigid transformations of $\cY$ due to the symmetry of the formulation. We provide proofs in the supplementary material.

The \textbf{reconstruction term} $\mathcal{E}_\mathrm{rec}$ encourages the reconstructed keypoints $\mXhI$ to be best aligned with the given keypoints of $\cY$ after they have been reordered via the permutation $\mP$ (we apply the same in a symmetric manner for reconstructing keypoints $\mYhJ$ from keypoints of $\cX$  reordered via $\mP^\top$):
\begin{align}
    \mathcal{E}_\mathrm{rec} = \frac{1}{n\dcY}\bigl\|\mXhI{-}\mP\mYJ\bigr\|_F{+}\frac{1}{n\dcX}\bigl\|\mYhJ{-}\mP^\top\mXI\bigr\|_F,
\end{align}
where $\dcX$ and $\dcY$ denote the diameter of $\cX$ and $\cY$, \ie the maximum among the geodesic distances between all pairs of vertices for each shape, respectively.

The \textbf{deformation term} $\mathcal{E}_\mathrm{def}$ favours reconstructions $\mXh$  which have a similar local geometry to $\cX$ as it was encoded in its PLBO $\DeltaXproj$ (again, we apply the same in a symmetric manner for  favouring reconstructions $\mYh$ with a similar local geometry to $\cY$ as encoded by $\DeltaYproj$):
\begin{align}
    \mathcal{E}_\mathrm{def} = \frac{1}{\absmX\dcY}\bigl\|\DeltaXproj\mXh\bigr\|_F+\frac{1}{\absmY\dcX}\bigl\|\DeltaYproj\mYh\bigr\|_F.
    \label{eq:deformTerm}
\end{align}
Fig.~\ref{fig:lbovsproj} shows how the use of the PLBO affects $\mXh$.  

The \textbf{orientation-aware term} 
$\mathcal{E}_\mathrm{ori}$ preserves the surface orientation in the correspondence through extrinsic orientation-aware feature maps $\vhX$ and $\vhY$:
\begin{align}
    \label{eq:term-ori-aware}
    \mathcal{E}_\mathrm{ori} = \frac{1}{n}\bigl\| \vhXI-\mP\vhYJ \bigr\|_F.
\end{align}
Even though the PLBO uses extrinsic information, it is still agnostic to intrinsic symmetries, such as left-right flips, due to the in-built invariance under the group $\rmE{3}$, which includes mirror-symmetries.
Thus, we propose the orientation-aware regularisation term to disambiguate such intrinsic symmetries. 
Similar solutions were developed previously for the functional maps framework~\cite{ren2018continuous}. 
Our orientation-aware feature map is defined as 
\begin{equation}\label{eq:orientation-feature}
\vhX=
\begin{pmatrix}
\bigl\langle(\nabla\vfX)_1~~~\times(\nabla\vgX)_1~~~,\vnX_1\bigr\rangle \\
\vdots \\
\bigl\langle(\nabla\vfX)_{\absmX} \times(\nabla\vgX)_{\absmX},\vnX_{\absmX}\bigr\rangle 
\end{pmatrix},
\end{equation}
where $\vnX_i\in\bbR^3$ are the unit outer normals at the $i$-th vertex  of $\cX$, and $\vfX$ and $\vgX$ are two distinct scalar fields (see Sec.~\ref{subsec:implementation} and supplementary material for details).
The outer product of their normalised gradient fields convey their orientation through the right-hand rule. Thus, the resulting feature field $\vhX$ implicitly encodes such orientation information. Comparing the features on both surfaces $\cX$ and $\cY$ allows us to disambiguate intrinsic symmetries. 

At the same time, $\vhX$ maintains (proper) rotation invariance. 
This follows directly from the fact that the gradient fields $\nabla\vfX,\nabla\vgX$ and normal field $\vnX$ are both rotation-equivariant, leaving their inner product unchanged. %

\vspace{-1mm}
\section{Experiments} \label{sec:experiments}

In this section, we  qualitatively and quantitatively evaluate our proposed method on several shape matching datasets (including TOSCA \cite{Bronstein:2008:NGN:1462123}, SMAL \cite{Zuffi:CVPR:2017}, SHREC20 Non-Isometric \cite{Dyke:2020:track.b} and the DeformingThings4D-Matching dataset \cite{magnet2022smooth}), and study its performance in terms of accuracy, global optimality and global scaling dependency. 

\vspace{-0.5mm}
\subsection{Evaluation Metric}
\textbf{Accuracy.} We evaluate the correspondence accuracy according to the Princeton benchmark protocol \cite{Kim11}, which reports the \emph{percentage of correct keypoints (PCK)} for the given sparse control points. For dense methods \cite{roetzer2022scalable}, we only consider this subset of sparse keypoints to enable a direct comparison of accuracy.

\textbf{Optimality.} We use the \emph{relative optimality gap} to measure the degree of optimality. It is defined as
\begin{equation}
    \text{gap} = \Bigl| \frac{ \overline{\text{obj}} - \underline{\text{obj}} }{\overline{\text{obj}}} \Bigr|,
    \label{eq:def_relGap}
\end{equation}
where $\overline{\text{obj}}$ is the best upper bound, and $\underline{\text{obj}}$ the best lower bound of the objective. 

\vspace{-0.5mm}
\subsection{Implementation Details}\label{subsec:implementation}

\textbf{Optimisation.}
The optimisation of all MIP-based methods was performed using MOSEK 10.0.35 \cite{mosek} on a desktop computer with AMD Ryzen 9 5950X 16-core processor and 64GB RAM. 
A time budget of 1h is set for each problem instance and its global optimum is considered to be found if the relative gap is reduced below $10^{-2}$.

\textbf{Parameters.}
We set $\lambda_\mathrm{def}=5$ and $\lambda_\mathrm{ori}=0.025$ in Eqn.~\eqref{eq:optimization-problem} and choose two entries with different frequencies from the scale-invariant wave kernel signature \cite{aubry2011wks} as $\vfX$ and $\vgX$ for the orientation-aware feature defined in Sec.~\ref{subsec:sparse-matching}. 
We coin the full formulation in Eqn.~\eqref{eq:optimization-problem} as \emph{Ours} and additionally show results without the orientation term $\mathcal{E}_\mathrm{ori}$ as \emph{Ours w/o Ori.}. Moreover, we note that dropping either $\cEdef$ or $\cErec$ will lead to a linear assignment problem for the permutation $\mP$ and a failed reconstruction in $\mXh,~\mYh$.

\textbf{Solution Pruning.}\label{sec:pruning}
We employ the same solution pruning strategy as  done in \cite{bernard2020mina} to reduce the search space. 
For each keypoint the pruning process keeps $k$ potential matches based on the similarity of the histogram of geodesics at each keypoint, see \cite{bernard2020mina} for details. 
In our experiments, we apply this pruning with $k=11$ to all methods that admit an initial search space reduction, namely to our method, MINA and PMSDP.

\vspace{-0.5mm}
\subsection{Competing Methods}

We compare against state-of-the-art matching methods which exhibit a global optimisation flavour.

\textbf{\smcomb}~is a combinatorial solver for the elastic shape matching formalism~\cite{windheuser2011geometrically} proposed by Roetzer et al.~\cite{roetzer2022scalable}.
It solves an integer linear program to find triangle-triangle matchings. 
This approach heavily relies on the discretisation of both shapes and might fail to produce a feasible solution entirely in cases of large triangulation discrepancies.

\textbf{PMSDP} is a convex semidefinite programming relaxation approach proposed by Maron et al.~\cite{maron2016point}.
It uses the LBO eigenbases as a high-dimensional feature embedding in the procrustes matching problem and is therefore an intrinsic and global scaling invariant formalism. 
Its implementation uses a pruning strategy to rule out unlikely correspondences using the average geodesic  distance (AGD) descriptor \cite{Kim11}. We found this to be too aggressive and often pruning the correct ground truth matches, thus, leading to poor performance. It is reported in Fig.~\ref{fig:pck_tosca} under \emph{PMSDP vanilla}. 
Hence, we replace its pruning strategy with ours, as discussed in Sec.~\ref{subsec:implementation}, while keeping all other default options, and coin it as \emph{PMSDP tuned}.

\textbf{MINA.}
This method proposed by Bernard et al. \cite{bernard2020mina} is the closest to ours among all baselines. 
Although MINA also employs a MIP model for sparse deformable shape matching problems, the key difference is that their approach is extrinsic due to its piece-wise affine deformation of the source to the target shape that is combined with a global rigid transformation, which in turn makes it computationally more expensive. 
 Note that MINA originally uses an aggressive search space pruning which we found to prune a large portion of correct matchings. Thus, we use a more conservative pruning by allowing twice as many potential matches as MINA.

\vspace{-0.5mm}
\subsection{Isometric Shape Matching}

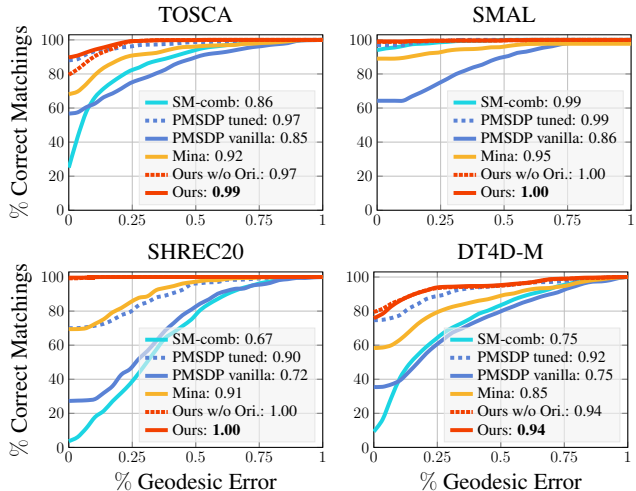
\begin{figure}[t]
    \centering
    \begin{tabular}{cc}
    \hspace{-1.2cm}
         \newcommand{\pckLineWidth}{3pt}
\newcommand{\plotWidth}{\columnwidth}
\newcommand{\plotHeight}{0.75\columnwidth}
\newcommand{\pckTitle}{TOSCA}
\definecolor{cPLOT0}{RGB}{28,213,227}
\definecolor{cPLOT1}{RGB}{80,150,80}
\definecolor{cPLOT2}{RGB}{90,130,213}
\definecolor{cPLOT3}{RGB}{247,179,43}
\definecolor{cPLOT5}{RGB}{242,64,0}

\pgfplotsset{%
    label style = {font=\large},
    tick label style = {font=\large},
    title style =  {font=\LARGE},
    legend style={  fill= gray!10,
                    fill opacity=0.6, 
                    font=\large,
                    draw=gray!20, %
                    text opacity=1}
}
\begin{tikzpicture}[scale=0.5, transform shape]
	\begin{axis}[
		width=\plotWidth,
		height=\plotHeight,
		grid=major,
		title=\pckTitle,
		legend style={
			at={(0.97,0.03)},
			anchor=south east,
			legend columns=1},
		legend cell align={left},
		ylabel={{\LARGE$\%$ Correct Matchings}},
		xmin=0,
        xmax=1,
        ylabel near ticks,
        xtick={0, 0.25, 0.5, 0.75, 1},
        ymin=0,
        ymax=103,
        ytick={0, 20, 40, 60, 80, 100},
	]
    \addplot [color=cPLOT0, smooth, line width=\pckLineWidth]
    table[row sep=crcr]{%
0 25\\
0.034483 42.7982\\
0.068966 57.7523\\
0.10345 66.3761\\
0.13793 71.5138\\
0.17241 75.4128\\
0.2069 78.8532\\
0.24138 81.8807\\
0.27586 83.7156\\
0.31034 86.1927\\
0.34483 87.7064\\
0.37931 88.8991\\
0.41379 90.367\\
0.44828 91.8807\\
0.48276 93.3486\\
0.51724 94.7706\\
0.55172 95.7339\\
0.58621 96.2844\\
0.62069 96.9266\\
0.65517 97.3394\\
0.68966 97.844\\
0.72414 98.3028\\
0.75862 98.578\\
0.7931 99.0367\\
0.82759 99.0826\\
0.86207 99.1284\\
0.89655 99.7248\\
0.93103 99.9083\\
0.96552 99.9541\\
1 100\\
    };
    \addlegendentry{\textcolor{black}{\smcomb: 0.86}}
    \addplot [color=cPLOT2, dashed, smooth, line width=\pckLineWidth]
    table[row sep=crcr]{%
0 88.0464\\
0.034483 88.6724\\
0.068966 91.4988\\
0.10345 93.3373\\
0.13793 94.164\\
0.17241 94.7166\\
0.2069 95.6267\\
0.24138 96.0706\\
0.27586 96.8973\\
0.31034 97.2069\\
0.34483 97.2988\\
0.37931 97.5641\\
0.41379 97.7597\\
0.44828 98.3126\\
0.48276 98.6256\\
0.51724 98.7038\\
0.55172 98.7651\\
0.58621 98.9828\\
0.62069 98.9828\\
0.65517 98.9828\\
0.68966 99.061\\
0.72414 99.1393\\
0.75862 99.2958\\
0.7931 99.374\\
0.82759 99.4523\\
0.86207 99.687\\
0.89655 99.8435\\
0.93103 100\\
1 100\\
    };
    \addlegendentry{\textcolor{black}{PMSDP tuned: 0.97}}
    \addplot [color=cPLOT2, smooth, line width=\pckLineWidth]
    table[row sep=crcr]{%
0 56.73\\
0.034483 57.4342\\
0.068966 60.481\\
0.10345 62.6683\\
0.13793 66.3726\\
0.17241 68.5794\\
0.2069 71.4291\\
0.24138 74.5937\\
0.27586 76.5606\\
0.31034 78.2648\\
0.34483 80.4017\\
0.37931 82.4969\\
0.41379 85.3889\\
0.44828 87.3755\\
0.48276 88.9122\\
0.51724 90.6558\\
0.55172 91.7396\\
0.58621 92.3264\\
0.62069 93.0025\\
0.65517 93.947\\
0.68966 94.8972\\
0.72414 95.7802\\
0.75862 96.328\\
0.7931 97.0769\\
0.82759 97.552\\
0.86207 98.0885\\
0.89655 99.3573\\
0.93103 99.7764\\
0.96552 99.8435\\
1 100\\
    };
    \addlegendentry{\textcolor{black}{PMSDP vanilla: 0.85}}
    \addplot [color=cPLOT3, smooth, line width=\pckLineWidth]
    table[row sep=crcr]{%
0 68.1383\\
0.034483 69.6267\\
0.068966 74.0399\\
0.10345 79.4878\\
0.13793 83.6568\\
0.17241 86.4506\\
0.2069 88.9657\\
0.24138 90.6113\\
0.27586 91.3209\\
0.31034 91.8574\\
0.34483 92.3828\\
0.37931 93.6598\\
0.41379 94.8782\\
0.44828 95.7445\\
0.48276 95.929\\
0.51724 96.3173\\
0.55172 96.6133\\
0.58621 96.7612\\
0.62069 96.9177\\
0.65517 97.0742\\
0.68966 97.6917\\
0.72414 97.77\\
0.75862 98.1612\\
0.7931 98.3568\\
0.82759 98.7089\\
0.86207 99.1393\\
0.89655 99.687\\
0.93103 99.8826\\
0.96552 99.9218\\
1 100\\
    };
    \addlegendentry{\textcolor{black}{Mina: 0.92}}
    \addplot [color=cPLOT5, smooth, densely dotted, line width=\pckLineWidth]
    table[row sep=crcr]{%
0 79.5787\\
0.034483 82.7224\\
0.068966 86.8528\\
0.10345 90.4995\\
0.13793 92.9835\\
0.17241 94.8527\\
0.2069 97.5738\\
0.24138 99.0462\\
0.27586 99.3949\\
0.31034 99.6209\\
0.34483 99.8163\\
0.37931 99.8469\\
0.41379 99.8775\\
0.44828 99.8775\\
0.48276 99.8775\\
0.51724 99.9081\\
0.55172 100\\
1 100\\
    };
    \addlegendentry{\textcolor{black}{Ours w/o Ori.: 0.97}}
    \addplot [color=cPLOT5, smooth, line width=\pckLineWidth]
    table[row sep=crcr]{%
0 89.7377\\
0.034483 90.7826\\
0.068966 92.6409\\
0.10345 94.3062\\
0.13793 95.5914\\
0.17241 96.9603\\
0.2069 98.239\\
0.24138 99.1942\\
0.27586 99.3643\\
0.31034 99.5903\\
0.34483 99.8163\\
0.37931 99.8469\\
0.41379 99.8775\\
0.44828 99.8775\\
0.48276 99.8775\\
0.51724 99.9081\\
0.55172 100\\
1 100\\
    };
    \addlegendentry{\textcolor{black}{Ours: \textbf{0.99}}}
        
	\end{axis}
\end{tikzpicture}& 
    \hspace{-1.25cm}
         \newcommand{\pckLineWidth}{3pt}
\newcommand{\plotWidth}{\columnwidth}
\newcommand{\plotHeight}{0.75\columnwidth}
\newcommand{\pckTitle}{SMAL}
\definecolor{cPLOT0}{RGB}{28,213,227}
\definecolor{cPLOT1}{RGB}{80,150,80}
\definecolor{cPLOT2}{RGB}{90,130,213}
\definecolor{cPLOT3}{RGB}{247,179,43}
\definecolor{cPLOT5}{RGB}{242,64,0}

\pgfplotsset{%
    label style = {font=\large},
    tick label style = {font=\large},
    title style =  {font=\LARGE},
    legend style={  fill= gray!10,
                    fill opacity=0.6, 
                    font=\large,
                    draw=gray!20, %
                    text opacity=1}
}
\begin{tikzpicture}[scale=0.5, transform shape]
	\begin{axis}[
		width=\plotWidth,
		height=\plotHeight,
		grid=major,
		title=\pckTitle,
		legend style={
			at={(0.97,0.03)},
			anchor=south east,
			legend columns=1},
		legend cell align={left},
		xmin=0,
        xmax=1,
        ylabel near ticks,
        xtick={0, 0.25, 0.5, 0.75, 1},
        ymin=0,
        ymax=103,
        ytick={0, 20, 40, 60, 80, 100},
	]
    \addplot [color=cPLOT0, smooth, line width=\pckLineWidth]
    table[row sep=crcr]{%
0 94\\
0.034483 95\\
0.068966 95.5455\\
0.10345 96.4545\\
0.13793 97.4545\\
0.17241 97.7273\\
0.2069 98\\
0.24138 98.5455\\
0.27586 98.8182\\
0.31034 99.2727\\
0.34483 99.3636\\
0.37931 99.4545\\
0.41379 99.4545\\
0.44828 99.6364\\
0.48276 99.9091\\
0.51724 99.9091\\
0.55172 100\\
1 100\\
    };
    \addlegendentry{\textcolor{black}{\smcomb: 0.99}}
    \addplot [color=cPLOT2, dashed, smooth, line width=\pckLineWidth]
    table[row sep=crcr]{%
0 96.9091\\
0.034483 96.9091\\
0.068966 96.9091\\
0.10345 97\\
0.13793 98\\
0.17241 98.7273\\
0.2069 99\\
0.24138 99.5455\\
0.27586 99.5455\\
0.31034 99.9091\\
0.34483 99.9091\\
0.37931 100\\
1 100\\
    };
    \addlegendentry{\textcolor{black}{PMSDP tuned: 0.99}}
    \addplot [color=cPLOT2, smooth, line width=\pckLineWidth]
    table[row sep=crcr]{%
0 64.2727\\
0.034483 64.2727\\
0.068966 64.2727\\
0.10345 64.3636\\
0.13793 66.9091\\
0.17241 69.1818\\
0.2069 71.7273\\
0.24138 74.4545\\
0.27586 77.0909\\
0.31034 79.9091\\
0.34483 81.9091\\
0.37931 83.9091\\
0.41379 85.4545\\
0.44828 87.7273\\
0.48276 89.4545\\
0.51724 90.5455\\
0.55172 92.3636\\
0.58621 93.0909\\
0.62069 94.2727\\
0.65517 95.4545\\
0.68966 96.2727\\
0.72414 97.4545\\
0.75862 98.1818\\
0.7931 98.8182\\
0.82759 99.0909\\
0.86207 99.1818\\
0.89655 99.3636\\
0.93103 99.7273\\
0.96552 99.9091\\
1 100\\
    };
    \addlegendentry{\textcolor{black}{PMSDP vanilla: 0.86}}
    \addplot [color=cPLOT3, smooth, line width=\pckLineWidth]
    table[row sep=crcr]{%
0 89\\
0.034483 89\\
0.068966 89\\
0.10345 89.4545\\
0.13793 90.3636\\
0.17241 91.2727\\
0.2069 92.1818\\
0.24138 92.5455\\
0.27586 93.6364\\
0.31034 94.2727\\
0.34483 94.5455\\
0.37931 94.6364\\
0.41379 95\\
0.44828 95.6364\\
0.48276 95.7273\\
0.51724 96.4545\\
0.55172 97.0909\\
0.58621 97.5455\\
0.62069 97.6364\\
0.65517 97.7273\\
0.68966 97.7273\\
0.72414 97.7273\\
0.75862 97.7273\\
0.7931 97.7273\\
0.82759 97.7273\\
0.86207 97.7273\\
0.89655 97.7273\\
0.93103 97.7273\\
0.96552 97.7273\\
1 97.7273\\
    };
    \addlegendentry{\textcolor{black}{Mina: 0.95}}
    \addplot [color=cPLOT5, smooth, densely dotted, line width=\pckLineWidth]
    table[row sep=crcr]{%
0 98.8182\\
0.034483 98.8182\\
0.068966 98.8182\\
0.10345 98.8182\\
0.13793 99.1818\\
0.17241 99.1818\\
0.2069 99.4545\\
0.24138 99.4545\\
0.27586 99.5455\\
0.31034 99.6364\\
0.34483 99.6364\\
0.37931 99.6364\\
0.41379 99.6364\\
0.44828 99.6364\\
0.48276 99.6364\\
0.51724 99.8182\\
0.55172 99.8182\\
0.58621 99.8182\\
0.62069 100\\
1 100\\
    };
    \addlegendentry{\textcolor{black}{Ours w/o Ori.: 1.00}}
    \addplot [color=cPLOT5, smooth, line width=\pckLineWidth]
    table[row sep=crcr]{%
0 99.0909\\
0.034483 99.0909\\
0.068966 99.0909\\
0.10345 99.0909\\
0.13793 99.3636\\
0.17241 99.3636\\
0.2069 99.5455\\
0.24138 99.5455\\
0.27586 99.5455\\
0.31034 99.6364\\
0.34483 99.6364\\
0.37931 99.6364\\
0.41379 99.6364\\
0.44828 99.6364\\
0.48276 99.6364\\
0.51724 99.8182\\
0.55172 99.8182\\
0.58621 99.8182\\
0.62069 100\\
1 100\\
    };
    \addlegendentry{\textcolor{black}{Ours: \textbf{1.00}}}
        
	\end{axis}
\end{tikzpicture}\\
     \hspace{-1.2cm}
         \newcommand{\pckLineWidth}{3pt}
\newcommand{\plotWidth}{\columnwidth}
\newcommand{\plotHeight}{0.75\columnwidth}
\newcommand{\pckTitle}{SHREC20}
\definecolor{cPLOT0}{RGB}{28,213,227}
\definecolor{cPLOT1}{RGB}{80,150,80}
\definecolor{cPLOT2}{RGB}{90,130,213}
\definecolor{cPLOT3}{RGB}{247,179,43}
\definecolor{cPLOT5}{RGB}{242,64,0}

\pgfplotsset{%
    label style = {font=\large},
    tick label style = {font=\large},
    title style =  {font=\LARGE},
    legend style={  fill= gray!10,
                    fill opacity=0.6, 
                    font=\large,
                    draw=gray!20, %
                    text opacity=1}
}
\begin{tikzpicture}[scale=0.5, transform shape]
	\begin{axis}[
		width=\plotWidth,
		height=\plotHeight,
		grid=major,
		title=\pckTitle,
		legend style={
			at={(0.97,0.03)},
			anchor=south east,
			legend columns=1},
		legend cell align={left},
		ylabel={{\LARGE$\%$ Correct Matchings}},
        xlabel={\LARGE$\%$ Geodesic Error},
		xmin=0,
        xmax=1,
        ylabel near ticks,
        xtick={0, 0.25, 0.5, 0.75, 1},
		ymin=0,
        ymax=103,
        ytick={0, 20, 40, 60, 80, 100},
	]
    \addplot [color=cPLOT0, smooth, line width=\pckLineWidth]
    table[row sep=crcr]{%
0 3.6364\\
0.034483 6\\
0.068966 11.2727\\
0.10345 17.8182\\
0.13793 21.4545\\
0.17241 27.0909\\
0.2069 32\\
0.24138 37.2727\\
0.27586 43.6364\\
0.31034 50.5455\\
0.34483 57.4545\\
0.37931 64.1818\\
0.41379 67.6364\\
0.44828 71.2727\\
0.48276 75.0909\\
0.51724 80.5455\\
0.55172 84.1818\\
0.58621 87.0909\\
0.62069 89.8182\\
0.65517 92\\
0.68966 93.8182\\
0.72414 95.0909\\
0.75862 96.5455\\
0.7931 97.4545\\
0.82759 98.7273\\
0.86207 99.6364\\
0.89655 100\\
1 100\\
    };
    \addlegendentry{\textcolor{black}{\smcomb: 0.67}}
    \addplot [color=cPLOT2, dashed, smooth, line width=\pckLineWidth]
    table[row sep=crcr]{%
0 70\\
0.034483 70\\
0.068966 70.5455\\
0.10345 71.4545\\
0.13793 72.5455\\
0.17241 74.7273\\
0.2069 77.6364\\
0.24138 79.4545\\
0.27586 83.4545\\
0.31034 85.2727\\
0.34483 88\\
0.37931 89.0909\\
0.41379 90.7273\\
0.44828 92.5455\\
0.48276 95.8182\\
0.51724 96.3636\\
0.55172 96.9091\\
0.58621 97.0909\\
0.62069 98.3636\\
0.65517 98.3636\\
0.68966 98.5455\\
0.72414 98.9091\\
0.75862 99.2727\\
0.7931 99.4545\\
0.82759 99.6364\\
0.86207 99.8182\\
0.89655 100\\
1 100\\
    };
    \addlegendentry{\textcolor{black}{PMSDP tuned: 0.90}}
    \addplot [color=cPLOT2, smooth, line width=\pckLineWidth]
    table[row sep=crcr]{%
0 27.2727\\
0.034483 27.4545\\
0.068966 27.6364\\
0.10345 28.3636\\
0.13793 32\\
0.17241 36\\
0.2069 42.7273\\
0.24138 46\\
0.27586 52\\
0.31034 56.7273\\
0.34483 62.5455\\
0.37931 68.5455\\
0.41379 73.0909\\
0.44828 77.6364\\
0.48276 80.9091\\
0.51724 84.3636\\
0.55172 86.9091\\
0.58621 90.1818\\
0.62069 91.8182\\
0.65517 93.0909\\
0.68966 93.2727\\
0.72414 95.0909\\
0.75862 95.6364\\
0.7931 96.5455\\
0.82759 98.1818\\
0.86207 98.7273\\
0.89655 99.4545\\
0.93103 99.4545\\
0.96552 99.8182\\
1 100\\
    };
    \addlegendentry{\textcolor{black}{PMSDP vanilla: 0.72}}
    \addplot [color=cPLOT3, smooth, line width=\pckLineWidth]
    table[row sep=crcr]{%
0 69.4545\\
0.034483 69.4545\\
0.068966 70\\
0.10345 72.5455\\
0.13793 75.2727\\
0.17241 79.2727\\
0.2069 81.8182\\
0.24138 85.0909\\
0.27586 87.8182\\
0.31034 88.7273\\
0.34483 92.3636\\
0.37931 93.6364\\
0.41379 94.7273\\
0.44828 96.1818\\
0.48276 96.9091\\
0.51724 97.6364\\
0.55172 98.1818\\
0.58621 98.1818\\
0.62069 98.9091\\
0.65517 99.2727\\
0.68966 99.4545\\
0.72414 99.4545\\
0.75862 99.8182\\
0.7931 99.8182\\
0.82759 100\\
1 100\\
    };
    \addlegendentry{\textcolor{black}{Mina: 0.91}}
    \addplot [color=cPLOT5, smooth, densely dotted, line width=\pckLineWidth]
    table[row sep=crcr]{%
0 99.2727\\
0.034483 99.2727\\
0.068966 99.6364\\
0.10345 99.6364\\
0.13793 100\\
1 100\\
    };
    \addlegendentry{\textcolor{black}{Ours w/o Ori.: 1.00}}
    \addplot [color=cPLOT5, smooth, line width=\pckLineWidth]
    table[row sep=crcr]{%
0 99.6364\\
0.034483 99.6364\\
0.068966 99.6364\\
0.10345 99.6364\\
0.13793 100\\
1 100\\
    };
    \addlegendentry{\textcolor{black}{Ours: \textbf{1.00}}}
        
	\end{axis}
\end{tikzpicture}& 
    \hspace{-1.25cm}
         \newcommand{\pckLineWidth}{3pt}
\newcommand{\plotWidth}{\columnwidth}
\newcommand{\plotHeight}{0.75\columnwidth}
\newcommand{\pckTitle}{DT4D-M}
\definecolor{cPLOT0}{RGB}{28,213,227}
\definecolor{cPLOT1}{RGB}{80,150,80}
\definecolor{cPLOT2}{RGB}{90,130,213}
\definecolor{cPLOT3}{RGB}{247,179,43}
\definecolor{cPLOT5}{RGB}{242,64,0}

\pgfplotsset{%
    label style = {font=\large},
    tick label style = {font=\large},
    title style =  {font=\LARGE},
    legend style={  fill= gray!10,
                    fill opacity=0.6, 
                    font=\large,
                    draw=gray!20, %
                    text opacity=1}
}
\begin{tikzpicture}[scale=0.5, transform shape]
	\begin{axis}[
		width=\plotWidth,
		height=\plotHeight,
		grid=major,
		title=\pckTitle,
		legend style={
			at={(0.97,0.03)},
			anchor=south east,
			legend columns=1},
		legend cell align={left},
        xlabel={\LARGE$\%$ Geodesic Error},
		xmin=0,
        xmax=1,
        ylabel near ticks,
        xtick={0, 0.25, 0.5, 0.75, 1},
		ymin=0,
        ymax=103,
        ytick={0, 20, 40, 60, 80, 100},
	]
    \addplot [color=cPLOT0, smooth, line width=\pckLineWidth]
    table[row sep=crcr]{%
0 9.2778\\
0.034483 16.3333\\
0.068966 29.1667\\
0.10345 40.8889\\
0.13793 48.6111\\
0.17241 54.1667\\
0.2069 59.0556\\
0.24138 63.2778\\
0.27586 67.3333\\
0.31034 70.3889\\
0.34483 72.8333\\
0.37931 75.5\\
0.41379 78.5\\
0.44828 80.2222\\
0.48276 82.4444\\
0.51724 84.5556\\
0.55172 86.8333\\
0.58621 88.6111\\
0.62069 90.1111\\
0.65517 91.3333\\
0.68966 92.7778\\
0.72414 94.5556\\
0.75862 95.6111\\
0.7931 96.8889\\
0.82759 98\\
0.86207 98.6667\\
0.89655 99.2778\\
0.93103 99.6667\\
0.96552 99.9444\\
1 100\\
    };
    \addlegendentry{\textcolor{black}{\smcomb: 0.75}}
    \addplot [color=cPLOT2, dashed, smooth, line width=\pckLineWidth]
    table[row sep=crcr]{%
0 74.6667\\
0.034483 74.8889\\
0.068966 75.9444\\
0.10345 77.6667\\
0.13793 81.5556\\
0.17241 84.0556\\
0.2069 87.1111\\
0.24138 88.6111\\
0.27586 89.5556\\
0.31034 91.6667\\
0.34483 92.7222\\
0.37931 93.4444\\
0.41379 94\\
0.44828 94.3333\\
0.48276 94.8889\\
0.51724 95.3333\\
0.55172 95.6111\\
0.58621 96.2222\\
0.62069 96.7222\\
0.65517 97\\
0.68966 97.0556\\
0.72414 97.3333\\
0.75862 97.8333\\
0.7931 98.4444\\
0.82759 98.8333\\
0.86207 98.9444\\
0.89655 99.5\\
0.93103 99.8333\\
0.96552 99.8333\\
1 100\\
    };
    \addlegendentry{\textcolor{black}{PMSDP tuned: 0.92}}
    \addplot [color=cPLOT2, smooth, line width=\pckLineWidth]
    table[row sep=crcr]{%
0 35.4444\\
0.034483 35.5556\\
0.068966 36.9444\\
0.10345 39.3333\\
0.13793 44.2778\\
0.17241 49.6111\\
0.2069 55.1667\\
0.24138 59.6667\\
0.27586 63.6111\\
0.31034 67.4444\\
0.34483 69.6111\\
0.37931 72.2222\\
0.41379 74.5556\\
0.44828 76.8333\\
0.48276 78.7778\\
0.51724 80.9444\\
0.55172 82.8889\\
0.58621 84.9444\\
0.62069 86.5556\\
0.65517 88.3333\\
0.68966 89.7222\\
0.72414 91\\
0.75862 93.0556\\
0.7931 94.6111\\
0.82759 96.2778\\
0.86207 97.2778\\
0.89655 97.4444\\
0.93103 98.1667\\
0.96552 99.4444\\
1 100\\
    };
    \addlegendentry{\textcolor{black}{PMSDP vanilla: 0.75}}
    \addplot [color=cPLOT3, smooth, line width=\pckLineWidth]
    table[row sep=crcr]{%
0 58.3333\\
0.034483 58.6111\\
0.068966 59.9444\\
0.10345 63.2778\\
0.13793 67.7222\\
0.17241 72.3889\\
0.2069 76\\
0.24138 78.7778\\
0.27586 80.8889\\
0.31034 82.5556\\
0.34483 83.7222\\
0.37931 84.9444\\
0.41379 86.1111\\
0.44828 86.9444\\
0.48276 88.3333\\
0.51724 89.6667\\
0.55172 90.3889\\
0.58621 91.8333\\
0.62069 92.8333\\
0.65517 93.3889\\
0.68966 93.5556\\
0.72414 94.2222\\
0.75862 94.9444\\
0.7931 96.3333\\
0.82759 97.3333\\
0.86207 98.1667\\
0.89655 98.6667\\
0.93103 99.3333\\
0.96552 99.7222\\
1 100\\
    };
    \addlegendentry{\textcolor{black}{Mina: 0.85}}
    \addplot [color=cPLOT5, smooth, densely dotted, line width=\pckLineWidth]
    table[row sep=crcr]{%
0 79.4915\\
0.034483 81.2994\\
0.068966 84.4068\\
0.10345 86.4407\\
0.13793 88.6441\\
0.17241 90.678\\
0.2069 92.2599\\
0.24138 93.3333\\
0.27586 93.7853\\
0.31034 94.0678\\
0.34483 94.1808\\
0.37931 94.2373\\
0.41379 94.2938\\
0.44828 94.5198\\
0.48276 94.7458\\
0.51724 95.0847\\
0.55172 95.5932\\
0.58621 96.2712\\
0.62069 96.7797\\
0.65517 97.6271\\
0.68966 98.4181\\
0.72414 98.8701\\
0.75862 98.9831\\
0.7931 99.209\\
0.82759 99.3785\\
0.86207 99.435\\
0.89655 99.548\\
0.93103 99.887\\
0.96552 100\\
1 100\\
    };
    \addlegendentry{\textcolor{black}{Ours w/o Ori.: 0.94}}
    \addplot [color=cPLOT5, smooth, line width=\pckLineWidth]
    table[row sep=crcr]{%
0 75.8889\\
0.034483 78.3889\\
0.068966 82.7778\\
0.10345 86.1111\\
0.13793 88.6667\\
0.17241 90.6111\\
0.2069 92.1111\\
0.24138 93.7222\\
0.27586 94.2222\\
0.31034 94.3333\\
0.34483 94.5556\\
0.37931 94.7222\\
0.41379 94.7222\\
0.44828 94.8333\\
0.48276 95.1111\\
0.51724 95.5\\
0.55172 96\\
0.58621 96.3889\\
0.62069 96.9444\\
0.65517 97.8333\\
0.68966 98.6667\\
0.72414 98.9444\\
0.75862 99\\
0.7931 99.3333\\
0.82759 99.6111\\
0.86207 99.6667\\
0.89655 99.7778\\
0.93103 99.8889\\
0.96552 100\\
1 100\\
    };
    \addlegendentry{\textcolor{black}{Ours: \textbf{0.94}}}
        
	\end{axis}
\end{tikzpicture}
    \end{tabular}
    \caption{\textbf{PCK curves} on datasets TOSCA \cite{Bronstein:2008:NGN:1462123}, SMAL \cite{Zuffi:CVPR:2017}, SHREC20 \cite{Dyke:2020:track.b} and DT4D-M \cite{magnet2022smooth}. The values in the legends are \emph{Area-Under-the-Curves} (AUC $\uparrow$). Across all datasets our method consistently outperforms all other approaches. For SMAL and SHREC20 our method produces nearly perfect results. The orientation-aware term improves our performance in most cases. %
    }
    \label{fig:pck_tosca}
\end{figure}

\textbf{TOSCA}~\cite{Bronstein:2008:NGN:1462123} is a popular dataset for the task of isometric shape matching. It contains shapes of 9 different categories and for each category, multiple non-rigidly deformed shapes. In our experiment, the sparse keypoints provided by \cite{Kim11} with ground truth labels are considered for matching, leading to $71$ different isometric pairs with $21$ to $46$ keypoints. We run all experiments with a time budget of 1h\footnote{\smcomb~runs till its end due to the lack of time limitation.}.

Fig.~\ref{fig:pck_tosca} summarises the PCK for the keypoints and \linebreak Fig.~\ref{fig:box_tosca} (left) unveils the statistics of the final relative optimality gaps.  Fig.~\ref{fig:q-comparison} shows qualitative matching results. 
Our proposed method outperforms all baselines and is able to produce the lowest relative gaps. More specifically, ours successfully certifies $56$ out of the total $71$ matching pairs whereas MINA can only certify $8$ pairs (cf. Fig.~\ref{fig:teaser}c, Tab.~\ref{tab:optimal_pairs}). Note that these numbers are different to the ones reported in MINA~\cite{bernard2020mina} due to their substantially more aggressive search space pruning.
PMSDP computes sparse correspondences very fast due to its convex (relaxed) formulation of the original non-convex problem. However, it requires an additional post-processing step to project onto the space of permutations and the final results have larger relative optimality gaps and are less accurate.
\smcomb's final estimation also has larger optimality gaps, which results in bad matching performance. Consequently, it is only able to provide optimality certification for $13$ out of all $71$ pairs (cf. Tab.~\ref{tab:optimal_pairs}).

\begin{figure}[t!]
    \centering
    \begin{tabular}{cc}
        \hspace{-1cm}
         \makecell{\definecolor{cPLOT0}{RGB}{28,213,227}
\definecolor{cPLOT1}{RGB}{80,150,80}
\definecolor{cPLOT2}{RGB}{90,130,213}
\definecolor{cPLOT3}{RGB}{247,179,43}
\definecolor{cPLOT5}{RGB}{242,64,0}
\pgfplotsset{%
    label style = {font=\large},
    tick label style = {font=\large},
    title style =  {font=\Large},
    legend style={  fill= gray!10,
                    fill opacity=0.6, 
                    font=\large,
                    draw=gray!20, %
                    text opacity=1}
}
\begin{tikzpicture}[scale=0.5, transform shape]
    \begin{axis}[
		width=1\columnwidth,
		height=0.75\columnwidth,
		title=TOSCA: Statistics on Rel. Opt. Gaps ,
		boxplot/draw direction=y,
		grid=major,
		ymin=-5,
		ymax=160,
		xmin=0.65,
		xmax=4.35,
	    ytick={0,20, 40, 60, 80, 100, 120, 140, 160},
            extra y ticks=150,
            extra y tick labels={$\dots$},
            extra y tick style={grid=minor, grid style={dashed,black}
        },
        ymajorgrids=true,
		yticklabels={0,20, 40, 60, 80, 100, 120, 140,$\infty$},
		xtick={1,2,3,4,5},
		xticklabels={\smcomb, PMSDP, MINA, Ours, $\phantom{1}$},
		ylabel={Rel. Opt. Gap ($\%$)},
		every boxplot/.style={mark=x,every mark/.append style={mark size=5pt}},
		boxplotcolor/.style={color=#1,very thick,fill=#1!50,mark options={color=#1,fill=#1!70}},
            boxplot/box extend=0.4%
		]
		\addplot+[boxplotcolor=cPLOT0, boxplot prepared={
			lower whisker=0.000000,
			lower quartile=3.158760,
			median=8.619597,
			upper quartile=17.673926,
			upper whisker=43.521389
			}]
		coordinates {(0,160)};%
		\addplot+[boxplotcolor=cPLOT2, boxplot prepared={
			lower whisker=1.521102,
			lower quartile=24.386971,
			median=34.253678,
			upper quartile=47.732127,
			upper whisker=125.501425
			}]
		coordinates {};
		\addplot+[boxplotcolor=cPLOT3, boxplot prepared={
			lower whisker=0.000000,
			lower quartile=40.269156,
			median=48.197013,
			upper quartile=55.732163,
			upper whisker=72.718267
			}]
		coordinates {};
		\addplot+[boxplotcolor=cPLOT5, boxplot prepared={
			lower whisker=0.000000,
			lower quartile=0.000000,
			median=0.465987,
			upper quartile=0.894211,
			upper whisker=28.860368
			}]
		coordinates {};
	\end{axis}
\end{tikzpicture}}&
         \hspace{-1.1cm}
         \makecell{\newcommand{\scaleLinewWidth}{3pt}
\definecolor{cPLOT0}{RGB}{28,213,227}
\definecolor{cPLOT1}{RGB}{80,150,80}
\definecolor{cPLOT2}{RGB}{90,130,213}
\definecolor{cPLOT3}{RGB}{247,179,43}
\definecolor{cPLOT5}{RGB}{242,64,0}

\pgfplotsset{%
    label style = {font=\large},
    tick label style = {font=\large},
    title style =  {font=\Large},
    legend style={  fill= gray!10,
                    fill opacity=0.6, 
                    font=\large,
                    draw=gray!20, %
                    text opacity=1}
}
\begin{tikzpicture}[scale=0.5, transform shape]
	\begin{axis}[
		width=\columnwidth,
		height=0.75\columnwidth,
		grid=major,
		title=Scale Invariance,
		legend style={
			at={(0.97,0.97)},
			anchor=north east,
			legend columns=1},
		legend cell align={left},
        ylabel={{\large Mean Geo. Err.}},
        xlabel={Scale of Shape $\cY$},
        xmode=log,
        xmin=0.1,
        xmax=10,
        ylabel near ticks,
        ymin=-0.1,
        ymax=0.6,
        ytick={0, 0.2, 0.4, 0.6, 0.8, 1},
	]
    \addplot [color=cPLOT0, smooth, line width=\scaleLinewWidth]
    table[row sep=crcr]{%
0.1  0.32486\\
0.5  0.40884\\
1 0.35142\\
5 0.37922\\
10  0.44471\\
    };
    \addlegendentry{\textcolor{black}{\smcomb}}
    \addplot [color=cPLOT2, smooth, line width=\scaleLinewWidth]
    table[row sep=crcr]{%
0.1 0.11728\\
0.5 0.11728\\
1 0.11728\\
5 0.11728\\
10 0.11728\\
    };
    \addlegendentry{\textcolor{black}{PMSDP}}
    \addplot [color=cPLOT3, smooth, line width=\scaleLinewWidth]
    table[row sep=crcr]{%
0.1 0.39424\\
0.5 0.14412\\
1 0.097603\\
5 0.33173\\
10 0.34554\\
    };
    \addlegendentry{\textcolor{black}{MINA}}
    \addplot [color=cPLOT5, smooth, line width=\scaleLinewWidth]
    table[row sep=crcr]{%
0.1 7.9073e-05\\
0.5 0\\
1 0\\
5 0\\
10 0.0044366\\
    };
    \addlegendentry{\textcolor{black}{Ours}}
        
	\end{axis}
\end{tikzpicture}\vspace{-0.35cm}}
    \end{tabular}
    \caption{(Left) \textbf{Relative optimality gap statistics} on the TOSCA. The Inf gap reflects cases where \smcomb~\cite{roetzer2022scalable} fails to find a solution. Our proposed method produces the lowest and most concentrated relative optimality gap. (Right) Mean geodesic error for shape pairs with \textbf{different shape scales}. PMSDP \cite{maron2016point} and our method are the only ones which perform consistently across different shape scales, with ours being more accurate.
    }
    \label{fig:box_tosca}
\end{figure}
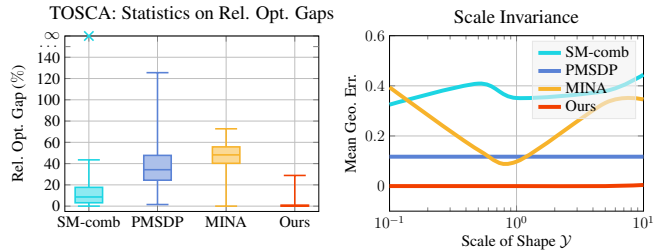

\begin{figure*}[!ht]
    \hspace{-1.2cm}
    \def\columnOne{dt4d-Punching015-BboyUprockStart035}
\def\columnTwo{dt4d-Standing2HMagicAttack01036-DancingRunningMan232}
\def\columnThree{dt4d-StandingReactLargeFromLeft010-SwingDancing031}
\def\columnFour{shrec20_bison_elephant_a}
\def\columnFive{shrec20_camel_a_cow}
\def\columnSix{shrec20_elephant_a_hippo}
\def\columnSeven{smal_00211799_ferrari_lion6}
\def\columnEight{smal_Davis_cow_00000_cow_alph4}
\def\columnNine{tosca_cat-0_cat-1}
\def\columnTen{tosca_wolf-0_wolf-2}
\def\markerfile{figs/red_rectangle_r_vthin_no_shad_less_red.png}
\def\heightQ{2cm}
\def\widthQ{1.77cm}
\def\hspaceCols{-0.5cm}
\begin{tabular}{lcccccccccc}%
        \setlength{\tabcolsep}{0pt} 
        & \circled{1} &  \circled{2} &  \circled{3} &  \circled{4} &  \circled{5} &  \circled{6} &  \circled{7} &  \circled{8} &  \circled{9} &  \circled{10} \\
        \rotatebox{90}{\hspace{0.35cm}Source}&
        \hspace{\hspaceCols}
        \includegraphics[height=\heightQ, width=\widthQ]{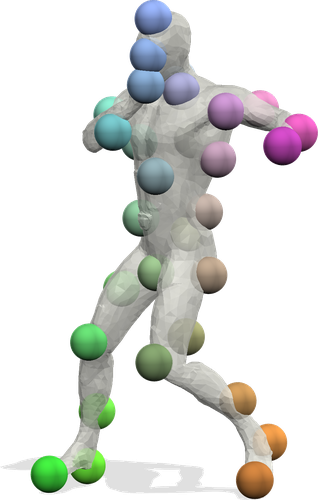}&
        \hspace{\hspaceCols}
        \includegraphics[height=\heightQ, width=\widthQ]{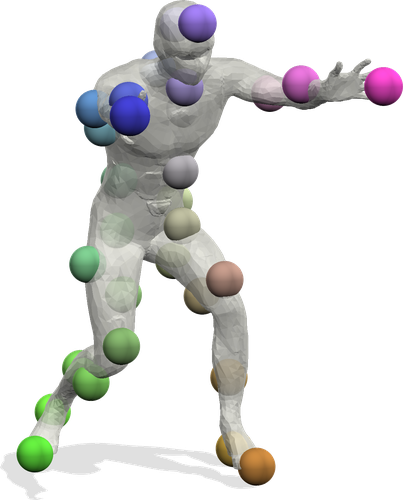}&
        \hspace{\hspaceCols}
        \includegraphics[height=\heightQ, width=\widthQ]{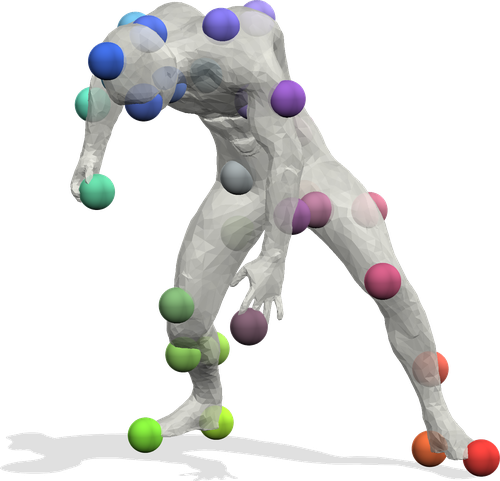}&
        \hspace{\hspaceCols}
        \includegraphics[height=\heightQ, width=\widthQ]{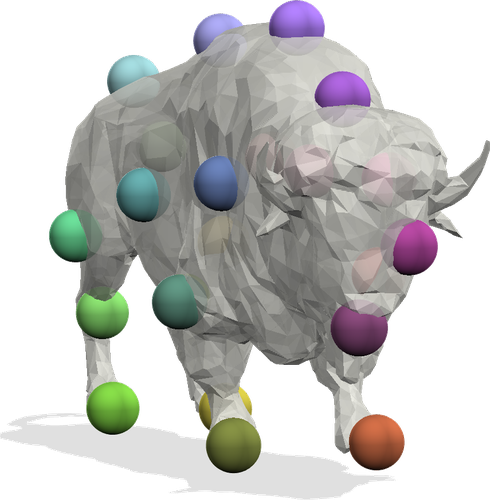}&
        \hspace{\hspaceCols}
        \includegraphics[height=\heightQ, width=\widthQ]{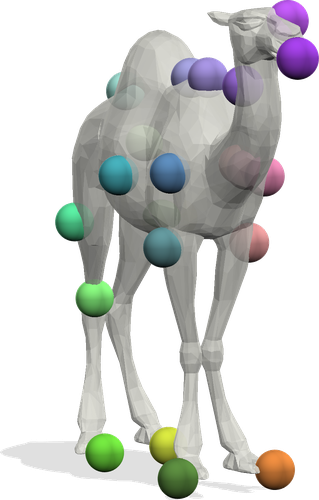}&
        \hspace{\hspaceCols}
        \includegraphics[height=\heightQ, width=\widthQ]{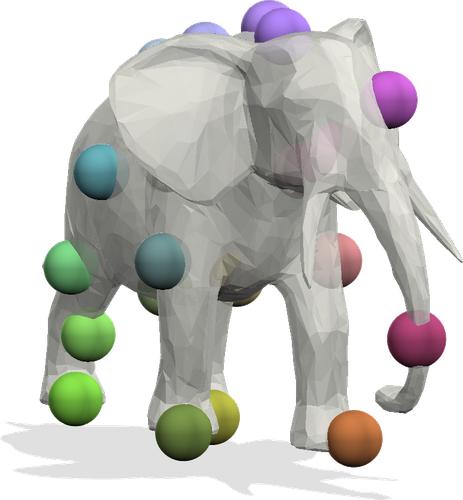}&
        \hspace{\hspaceCols}
        \includegraphics[height=\heightQ, width=\widthQ]{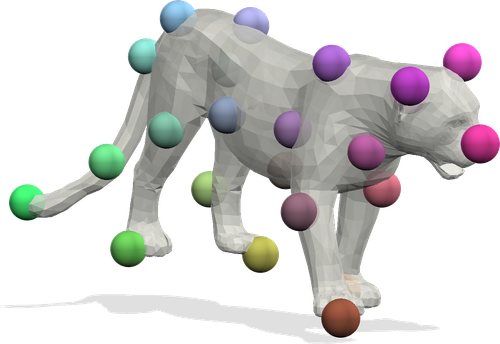}&
        \hspace{\hspaceCols}
        \includegraphics[height=\heightQ, width=\widthQ]{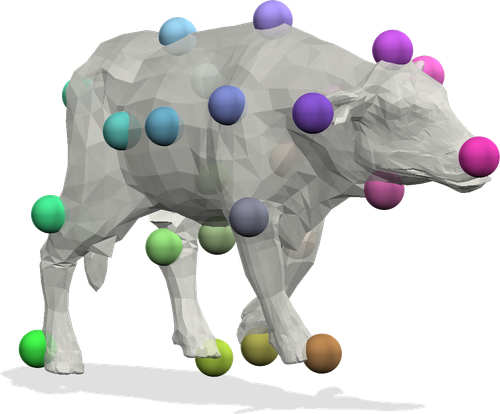}&
        \hspace{\hspaceCols}
        \includegraphics[height=\heightQ, width=\widthQ]{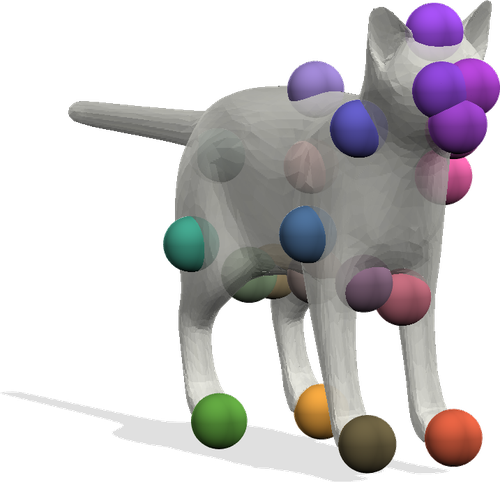}&
        \hspace{\hspaceCols}
        \includegraphics[height=\heightQ, width=\widthQ]{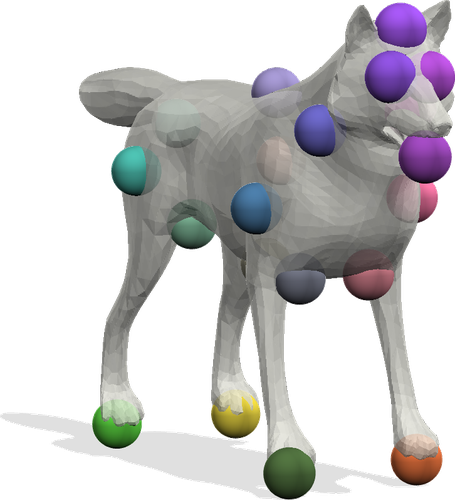}\\
        \rotatebox{90}{\hspace{0.35cm}\smcomb}&
        \hspace{\hspaceCols}
        \includegraphics[height=\heightQ, width=\widthQ]{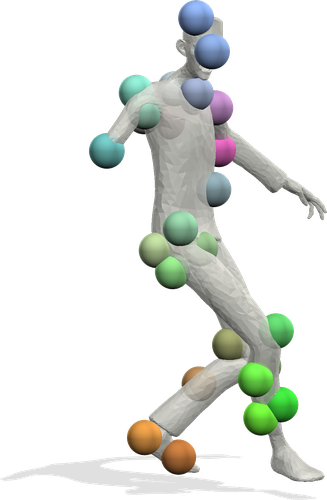}&
        \hspace{\hspaceCols}
        \includegraphics[height=\heightQ, width=\widthQ]{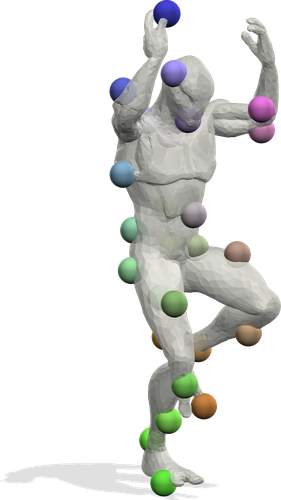}&
        \hspace{\hspaceCols}
        \begin{overpic}[height=\heightQ, width=\widthQ]{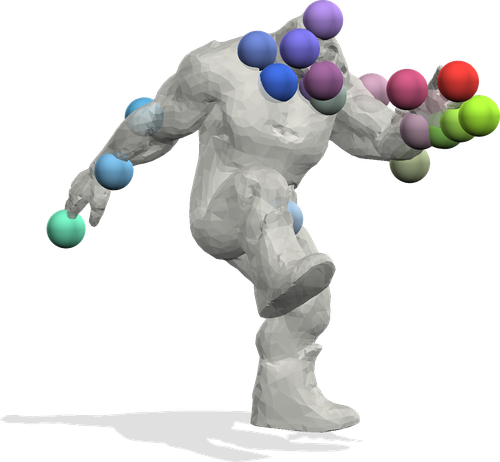}
            \put(-2,-10){\includegraphics[height=2cm, width=1.9cm, keepaspectratio=false]{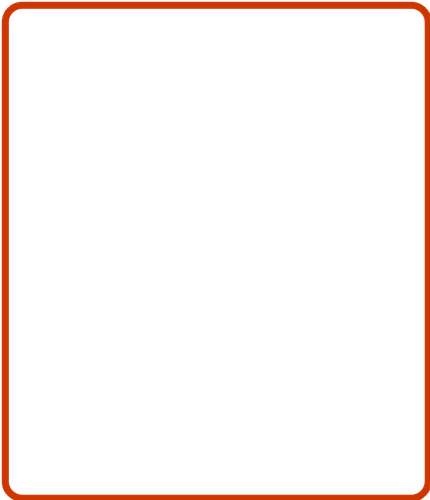}}
        \end{overpic}&
        \hspace{\hspaceCols}
        \includegraphics[height=\heightQ, width=\widthQ]{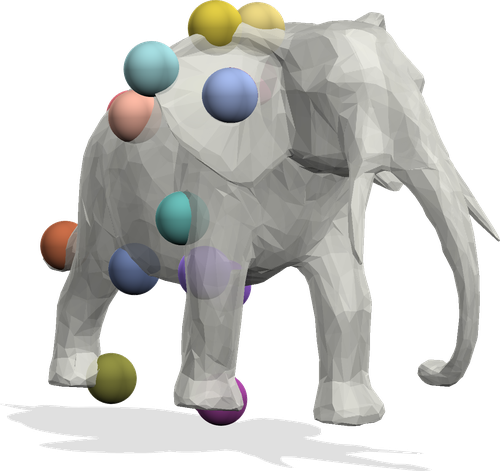}&
        \hspace{\hspaceCols}
        \includegraphics[height=\heightQ, width=\widthQ]{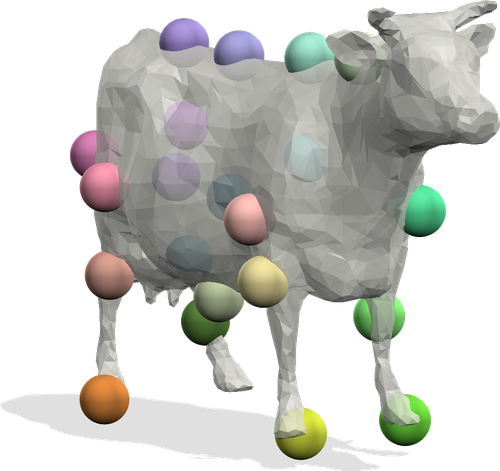}&
        \hspace{\hspaceCols}
        \begin{overpic}[height=\heightQ, width=\widthQ]{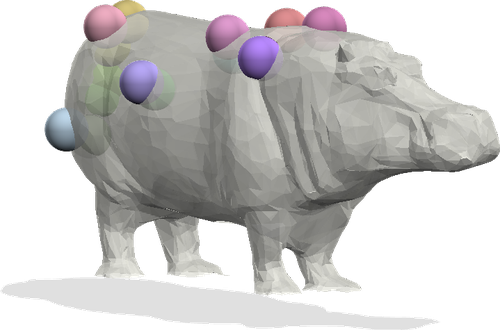}
            \put(0,-10){\includegraphics[height=2cm, width=1.8cm, keepaspectratio=false]{\markerfile}}
        \end{overpic}&
        \hspace{\hspaceCols}
        \includegraphics[height=\heightQ, width=\widthQ]{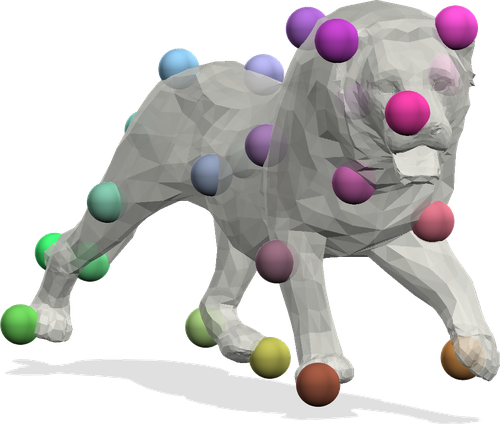}&
        \hspace{\hspaceCols}
        \includegraphics[height=\heightQ, width=\widthQ]{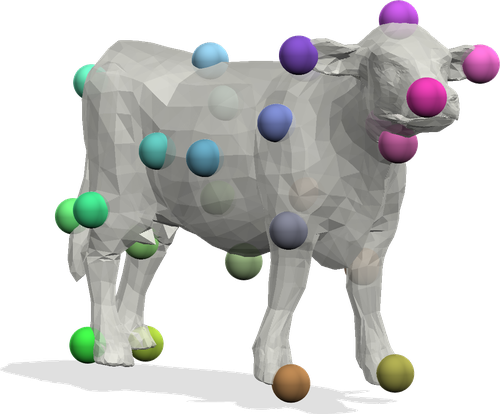}&
        \hspace{\hspaceCols}
        \begin{overpic}[height=\heightQ, width=\widthQ]{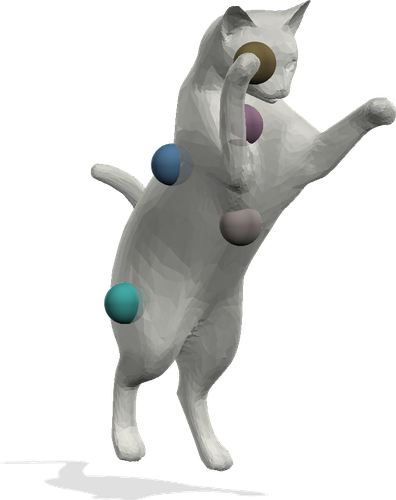}
            \put(0,-5){\includegraphics[height=2.2cm, width=1.8cm, keepaspectratio=false]{\markerfile}}
        \end{overpic}&
        \hspace{\hspaceCols}
        \includegraphics[height=\heightQ, width=\widthQ]{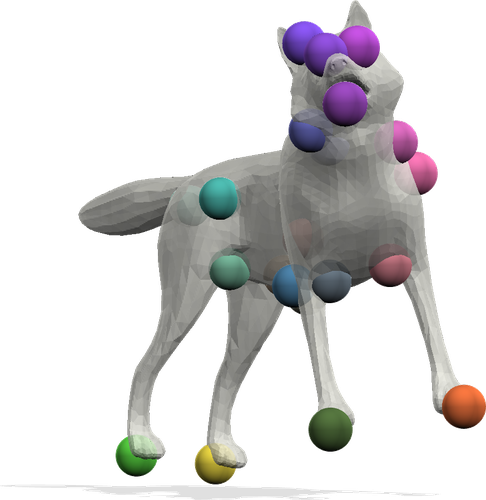}\\
        \rotatebox{90}{\hspace{0.35cm}PMSDP}&
        \hspace{\hspaceCols}
        \begin{overpic}[height=\heightQ, width=\widthQ]{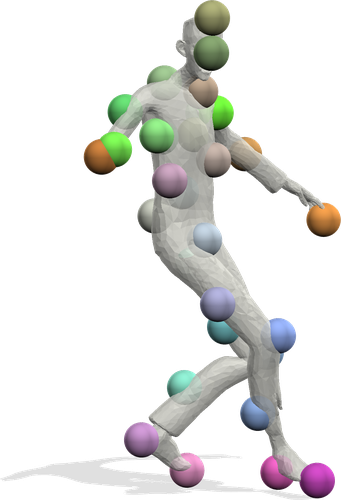}
            \put(5,-10){\includegraphics[height=2.3cm, width=1.55cm, keepaspectratio=false]{\markerfile}}
        \end{overpic}&
        \hspace{\hspaceCols}
        \includegraphics[height=\heightQ, width=\widthQ]{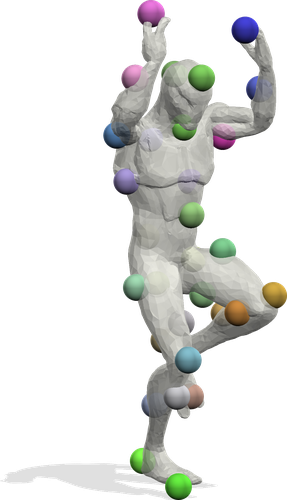}&
        \hspace{\hspaceCols}
        \includegraphics[height=\heightQ, width=\widthQ]{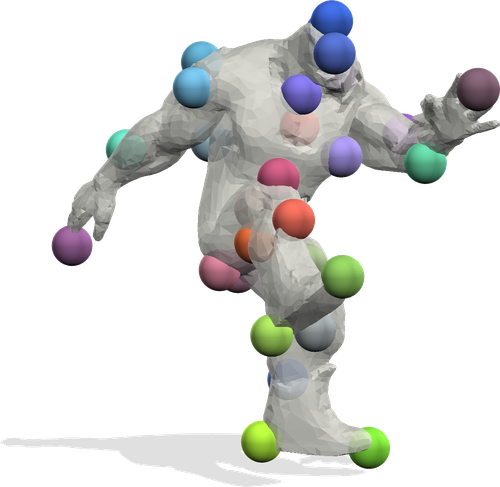}&
        \hspace{\hspaceCols}
        \begin{overpic}[height=\heightQ, width=\widthQ]{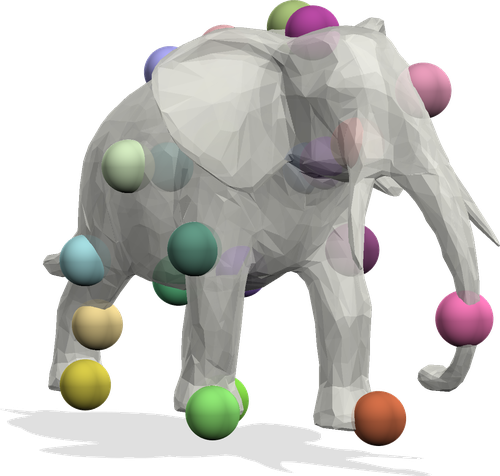}
            \put(0,-10){\includegraphics[height=2cm, width=1.85cm, keepaspectratio=false]{\markerfile}}
        \end{overpic}&
        \hspace{\hspaceCols}
        \includegraphics[height=\heightQ, width=\widthQ]{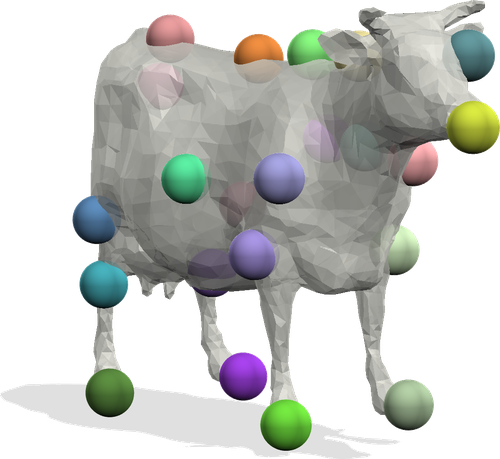}&
        \hspace{\hspaceCols}
        \begin{overpic}[height=\heightQ, width=\widthQ]{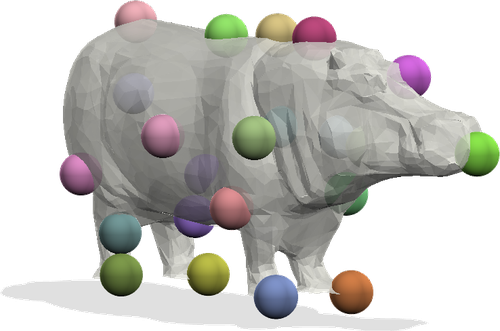}
            \put(0,-10){\includegraphics[height=2cm, width=1.8cm, keepaspectratio=false]{\markerfile}}
        \end{overpic}&
        \hspace{\hspaceCols}
        \includegraphics[height=\heightQ, width=\widthQ]{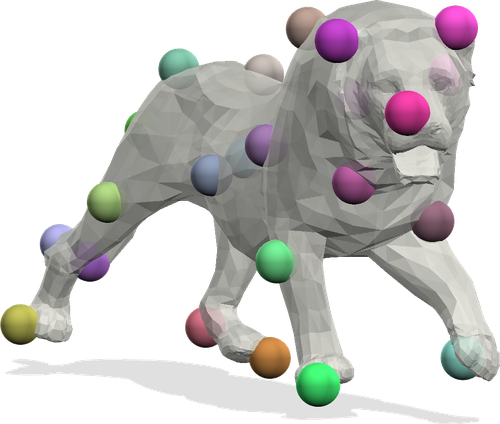}&
        \hspace{\hspaceCols}
        \includegraphics[height=\heightQ, width=\widthQ]{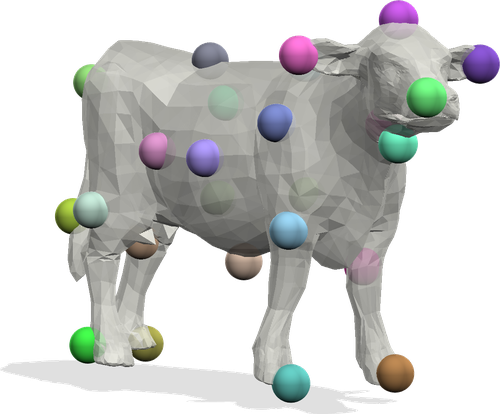}&
        \hspace{\hspaceCols}
        \includegraphics[height=\heightQ, width=\widthQ]{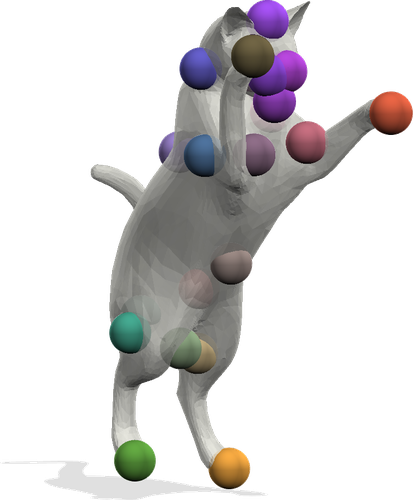}&
        \hspace{\hspaceCols}
        \includegraphics[height=\heightQ, width=\widthQ]{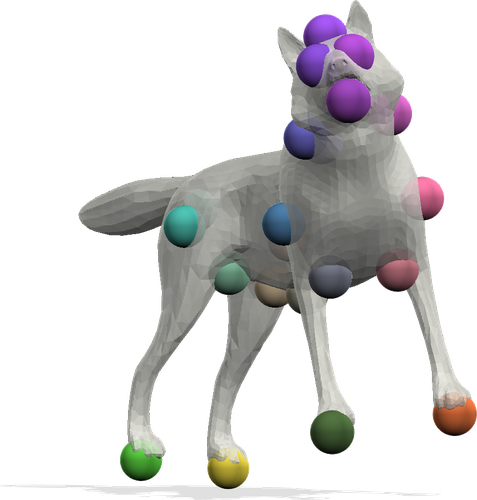}\\
        \rotatebox{90}{\hspace{0.4cm}MINA}&
        \hspace{\hspaceCols}
        \includegraphics[height=\heightQ, width=\widthQ]{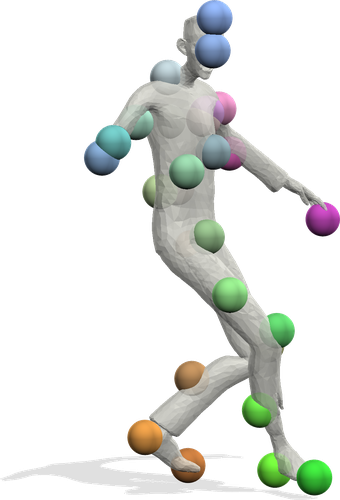}&
        \hspace{\hspaceCols}
        \includegraphics[height=\heightQ, width=\widthQ]{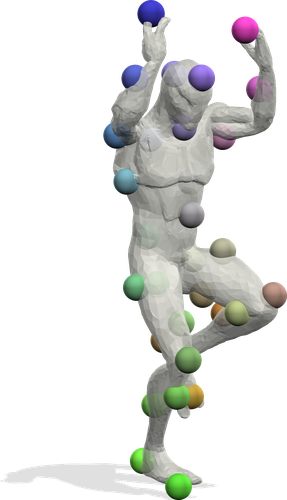}&
        \hspace{\hspaceCols}
        \includegraphics[height=\heightQ, width=\widthQ]{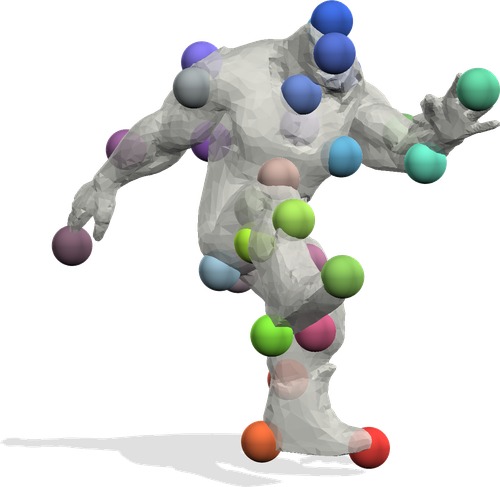}&
        \hspace{\hspaceCols}
        \begin{overpic}[height=\heightQ, width=\widthQ]{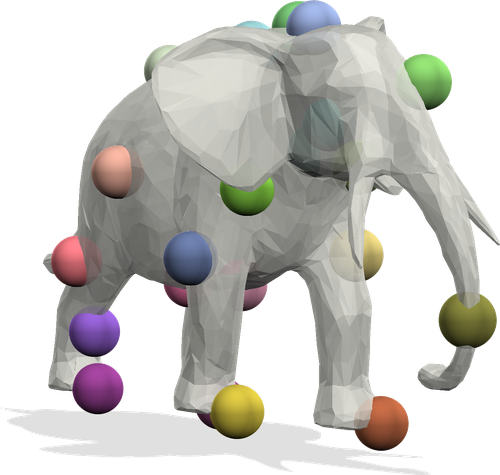}
            \put(0,-10){\includegraphics[height=2cm, width=1.85cm, keepaspectratio=false]{\markerfile}}
        \end{overpic}&
        \hspace{\hspaceCols}
        \includegraphics[height=\heightQ, width=\widthQ]{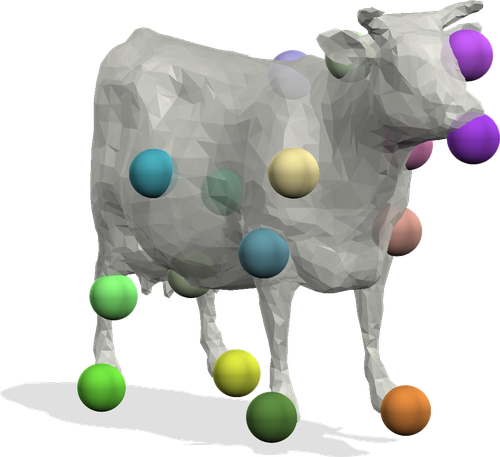}&
        \hspace{\hspaceCols}
        \includegraphics[height=\heightQ, width=\widthQ]{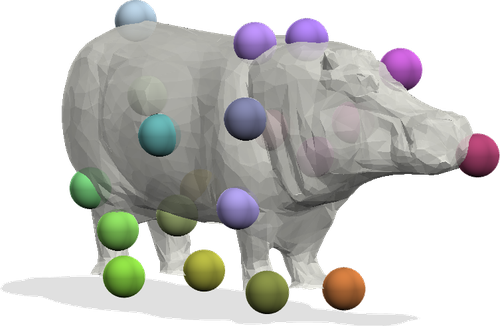}&
        \hspace{\hspaceCols}
        \includegraphics[height=\heightQ, width=\widthQ]{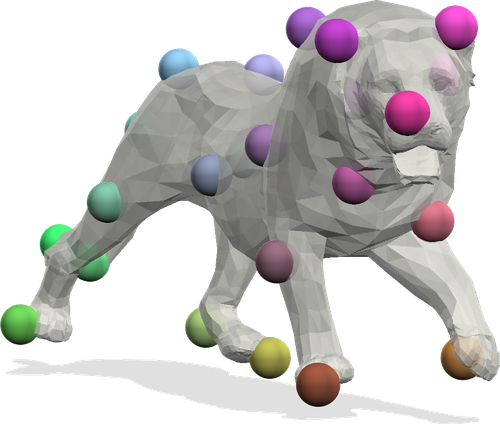}&
        \hspace{\hspaceCols}
        \begin{overpic}[height=\heightQ, width=\widthQ]{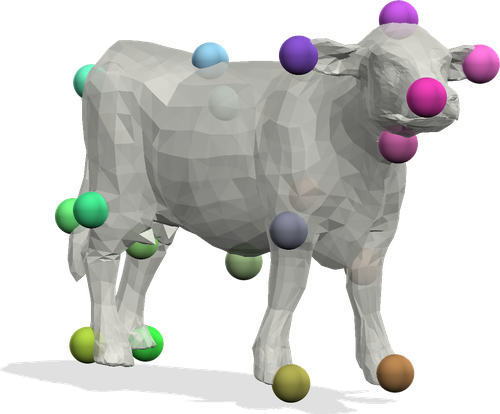}
            \put(0,-10){\includegraphics[height=2cm, width=1.9cm, keepaspectratio=false]{\markerfile}}
        \end{overpic}&
        \hspace{\hspaceCols}
        \includegraphics[height=\heightQ, width=\widthQ]{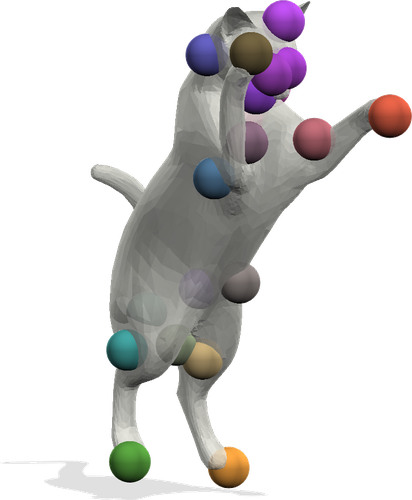}&
        \hspace{\hspaceCols}
        \includegraphics[height=\heightQ, width=\widthQ]{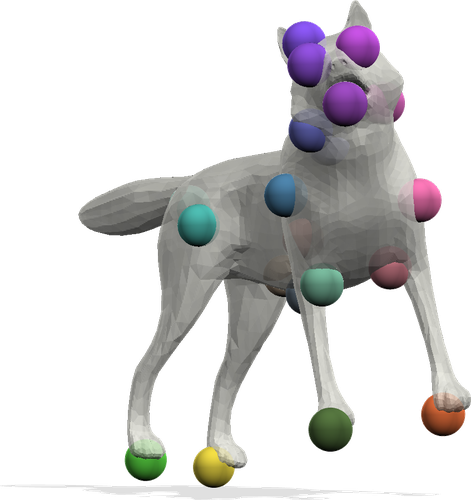}\\
        \rotatebox{90}{\hspace{0.5cm}Ours}&
        \hspace{\hspaceCols}
        \includegraphics[height=\heightQ, width=\widthQ]{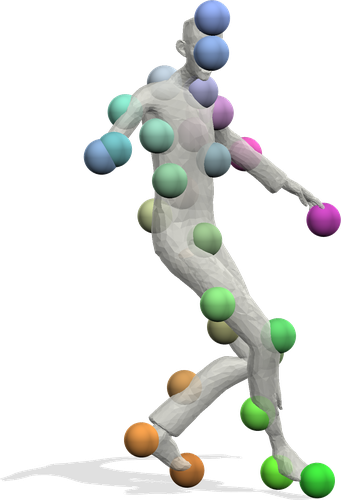}&
        \hspace{\hspaceCols}
        \includegraphics[height=\heightQ, width=\widthQ]{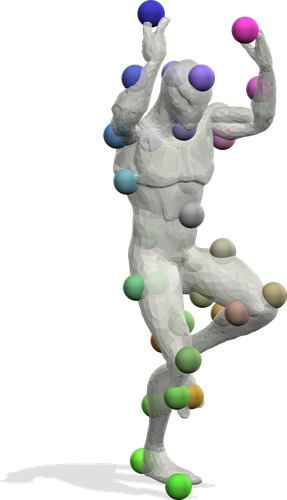}&
        \hspace{\hspaceCols}
        \includegraphics[height=\heightQ, width=\widthQ]{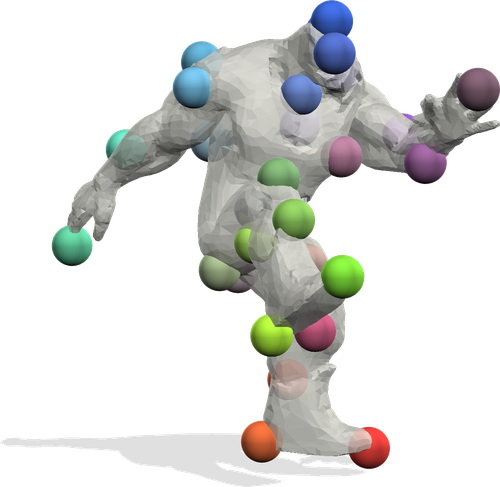}&
        \hspace{\hspaceCols}
        \includegraphics[height=\heightQ, width=\widthQ]{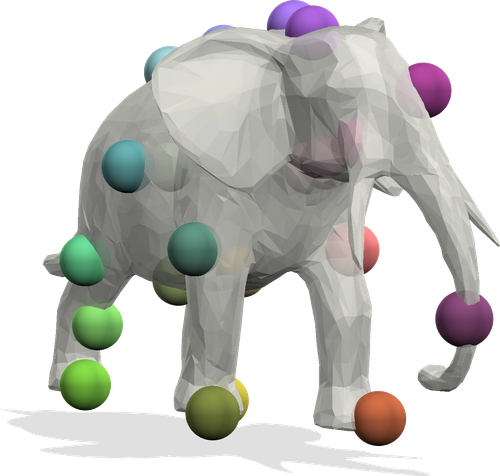}&
        \hspace{\hspaceCols}
        \includegraphics[height=\heightQ, width=\widthQ]{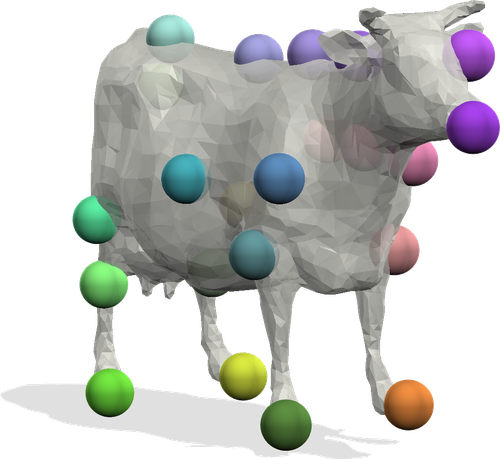}&
        \hspace{\hspaceCols}
        \includegraphics[height=\heightQ, width=\widthQ]{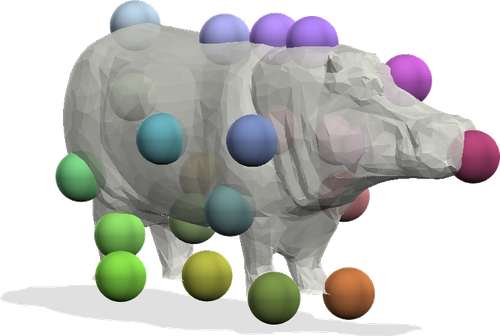}&
        \hspace{\hspaceCols}
        \includegraphics[height=\heightQ, width=\widthQ]{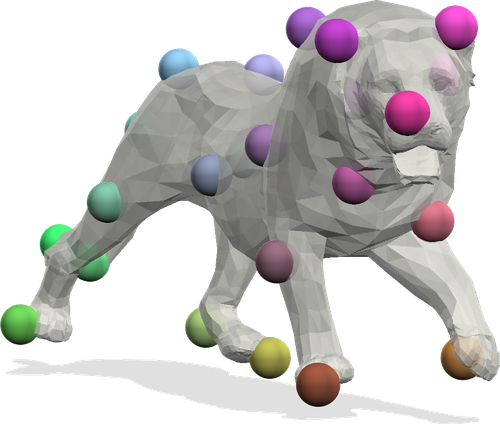}&
        \hspace{\hspaceCols}
        \includegraphics[height=\heightQ, width=\widthQ]{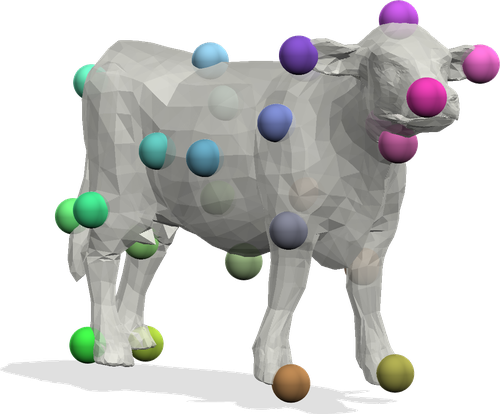}&
        \hspace{\hspaceCols}
        \includegraphics[height=\heightQ, width=\widthQ]{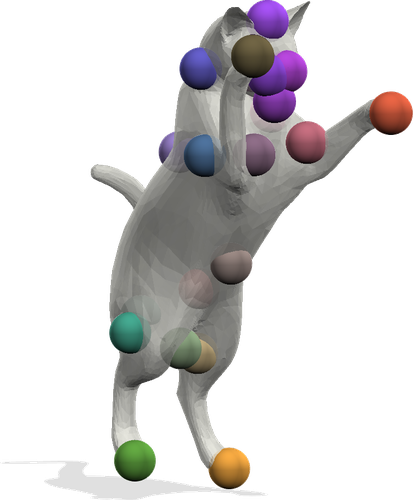}&
        \hspace{\hspaceCols}
        \includegraphics[height=\heightQ, width=\widthQ]{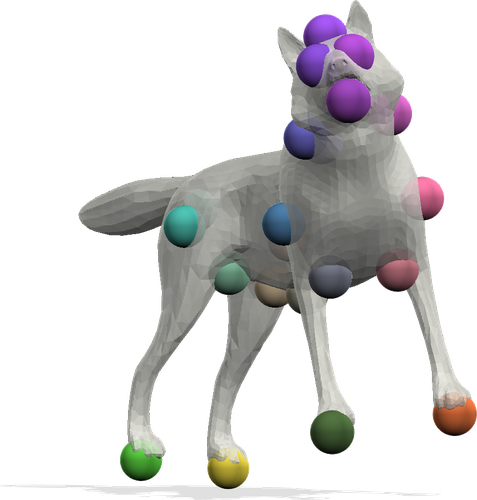}\\
    \end{tabular}
    \caption{\textbf{Qualitative results} on instances of DT4D-M \cite{magnet2022smooth} (\scircled{1}--\scircled{3}), SHREC20 \cite{Dyke:2020:track.b} (\scircled{4}--\scircled{6}), SMAL \cite{Zuffi:CVPR:2017} (\scircled{7},\scircled{8}) and TOSCA \cite{Bronstein:2008:NGN:1462123} (\scircled{9},\scircled{10}). While our approach consistently produces good results across all datasets, the other approaches have difficulties: \smcomb~\cite{roetzer2022scalable} (which is continuous and thus may match keypoints of shape $\cX$ to any point of shape $\cY$) cannot always find a feasible solution, which results in incomplete solutions (\scircled{3}, \scircled{6}, \scircled{9}); 
    PMSDP~\cite{maron2016point} especially struggles with non-isometric shape pairs (\scircled{1}, \scircled{4}, \scircled{6}), similarly as MINA~\cite{bernard2020mina} (\scircled{4}, \scircled{8}).}
    \label{fig:q-comparison}
\end{figure*}

\textbf{SMAL~\cite{Zuffi:CVPR:2017}} is a 3D animal dataset containing near-isometric shapes, such as foxes and dogs. 
For our experiments we use furthest point sampling (FPS) to select $25$  sparse keypoints,  consistently subsample the shapes to $5$k faces, and then transfer the selected keypoints to the lower resolution by nearest neighbour search.
This introduces noise into the keypoint positions while leaving the mesh topology consistent. 
Furthermore, the shapes are not pre-aligned and used `as is' in our experiments.

As shown in Fig.~\ref{fig:pck_tosca}, our proposed method works well in this setting and solves all matching instances to global optimality within the $1$h budget (cf. Tab.~\ref{tab:optimal_pairs}). 
\smcomb~ \cite{roetzer2022scalable} can produce high quality matchings due to the consistent meshing, which, however, is often not available in practice. 
MINA \cite{bernard2020mina} can handle near-isometric shapes and estimate accurate correspondences. However, it fails to close the relative optimality gap and thus cannot certify global optimality. 
The same holds for PMSDP \cite{maron2016point}. 

\vspace{-0.5mm}
\subsection{Non-Isometric Shape Matching}

\textbf{SHREC20~\cite{Dyke:2020:track.b}} contains $14$ non-isometric, high resolution animal shapes, ranging from dog and cow to giraffe and elephant. 
It provides consensus-based sparse keypoint correspondences which we use in our evaluation.
Note that these keypoints are manually selected by a human and an absolute ground truth does not exist. 
In total, $25$ shape pairs are randomly selected from the dataset.

Our proposed method outperforms all competitors on this dataset and finds the global optima for all $25$ matching pairs (cf. Tab.~\ref{tab:optimal_pairs}). 
Furthermore, all competing methods deteriorate heavily in this non-isometric setting. 
While it is much harder for MINA \cite{bernard2020mina} to deform the shapes, the high-dimensional embedding of LBO eigenfunctions used in PMSDP \cite{maron2016point} breaks for non-isometric shapes, hence their unsatisfactory performance. 
\smcomb~\cite{roetzer2022scalable} suffers from the large scale variation and the inconsistent triangulation present in the dataset.

\textbf{DeformingThings4D-Matching} (DT4D-M)
\cite{magnet2022smooth} is a recent dataset with humanoid shapes in various poses and inconsistent meshing. 
Dense but incomplete ground truth correspondences are provided for evaluation. 
We randomly created $60$ shape pairs, consistently sample $30$ keypoints and then (independently) subsample shapes to about $10$k faces. 
In this way, the selected keypoints are not in exact correspondence. 
This dataset often exhibits non-isometric and non-rigid deformations at the same time and poses a challenging task for all methods.

As shown in Fig.~\ref{fig:pck_tosca}, ours yields the best accuracy and obtains the most globally optimal solutions among all baselines despite the present difficulties (Tab.~\ref{tab:optimal_pairs}). 
Note that $\cEori$ is less helpful here than in TOSCA since it is designed to disambiguate symmetry based on the orientation-aware $\vh^{(\cdot)}$ which is more {accurate} for isometric shapes. DT4D-M (cf. TOSCA) has less isometric shapes (some without any symmetry, e.g.~\emph{prisoner}), so $\cEori$ has less impact (Fig.~\ref{fig:q-comparison}).
Besides the challenges mentioned above, many its shapes contain non-manifold structures which empirically causes more difficulties.

\vspace{-1mm}
\subsection{Global Scaling}

We showed the scale invariance of our method in \Cref{lemma:invariance-scale-rigid} in theory, but also validate our claims with experiments.
Shapes with different scales exist naturally, such as the dog-camel and bison-elephant shown in Figs.~\ref{fig:teaser}a $\&$ \ref{fig:q-comparison}, but many methods struggle with such large scale changes.
We randomly select five SHREC20 \cite{Dyke:2020:track.b} pairs to further study the effect of shape scale by rescaling one shape while fixing the other. 
More specifically, we fix the shape $\cX$ and rescale shape $\cY$ by a factor of $\{0.1, 0.5, 1, 5, 10\}$ respectively, to create $5$ instances with different scales of the same matching pair.
As shown in \Cref{lemma:invariance-scale-rigid} and Fig.~\ref{fig:box_tosca} (right), our proposed method is scale-invariant and solves all matching instances to global optimality independent of the shape scale, hence consistently achieving the lowest \emph{mean geodesic error}. While PMSDP \cite{maron2016point} is also scale-invariant, it fails under non-isometric deformation. All other baselines depend on shape scale and thus manifest an accuracy reduction upon scale changes.

\begin{table}
    \centering
    \begin{tabularx}{\columnwidth}{lcccc}
        \toprule
        & \smcomb & MINA & PMSDP &  Ours \\
        \toprule
         TOSCA (71)    & 13  & 8  &  4 & \textbf{56}  \\
         SMAL (44)     & 31  & 9  &  2 & \textbf{44}  \\
         SHREC20 (25)  & 0  & 5  &  0 & \textbf{25}  \\
         DT4D-M (60)   & 3  & 1  &  4 & \textbf{21} \\
         \hline
    \end{tabularx}
\caption{The \textbf{number of matching instances with certified global optimality} within 1h time budget. Numbers in  parentheses in the 1st column are the total number of matching pairs. \smcomb~\cite{roetzer2022scalable} does not allow to set a time budget, so we let it run until it terminates. 
}
\label{tab:optimal_pairs}
\end{table}

\subsection{Mesh Resolution}
The mesh resolution of shapes often plays an vital role in the scalability of matching methods, especially for global methods.  
The elastic matching model \cite{windheuser2011geometrically} targeted by \smcomb~\cite{roetzer2022scalable} is constructed in the space of product-surfaces, and therefore the number of binary variables to be optimised increases quadratically with the number of mesh triangles. MINA \cite{bernard2020mina} deforms each mesh triangle using an affine transformation, and in addition requires a global rotation matrix that drastically increases its runtime due to the involved discretisation. In our approach the continuous variables $\mXh, \mYh$  increase linearly with the number of mesh vertices while the number of  binary variables $\mP$ remains unchanged. As we demonstrate in Fig.~\ref{fig:teaser}d, our approach scales much better compared to \smcomb~and MINA, which stems from our more efficient matching formalism (e.g.~without rotation matrices opposed to MINA, or without a quadratic number of binary variables opposed to \smcomb).

\begin{figure}[t!]
    \centering
    \begin{tabular}{cc}
         \hspace{-1.1cm}
         \makecell{\def\columnOne{dt4d-InvertedDoubleKickToKipUp028-StandingReactLargeFromLeft021_cloud}
\def\columnTwo{dt4d-RifleTurnAndKick037-Running041_cloud}
\def\columnThree{dt4d-Standing2HMagicAttack01020-StandingCoverTurn030_cloud}
\def\columnFour{shrec20_elephant_a_hippo_cloud}
\def\columnFive{shrec20_hippo_rhino_cloud}
\def\columnSix{smal_00211799_ferrari_muybridge_132_133_07_cloud}
\def\heightPC{2.1cm}
\def\widthPC{1.9cm}
\def\hspaceColsPC{-0.5cm}
\begin{tabular}{cccccc}%
        \setlength{\tabcolsep}{0pt}
        \hspace{\hspaceColsPC}
        \includegraphics[height=\heightPC, width=\widthPC]{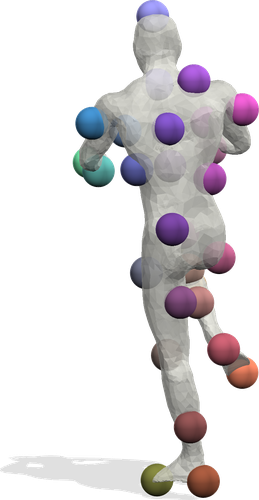}&
        \hspace{\hspaceColsPC}
        \includegraphics[height=\heightPC, width=\widthPC]{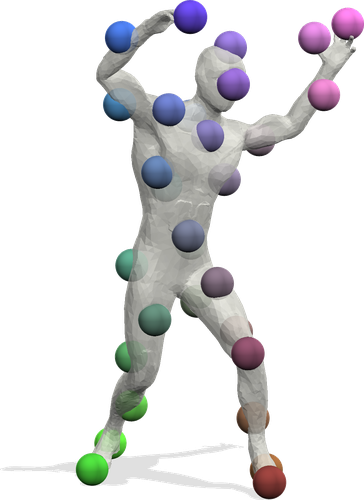}&
        \hspace{\hspaceColsPC}
        \includegraphics[height=\heightPC, width=\widthPC]{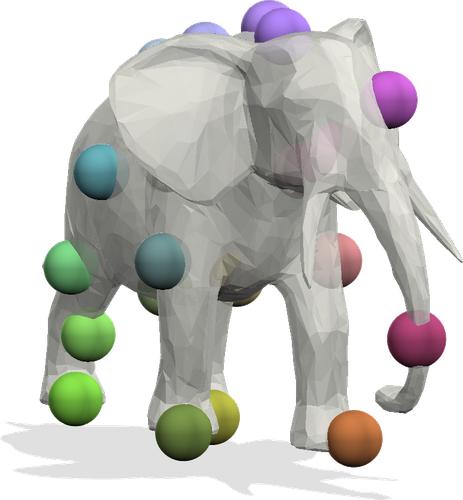}&\\
        \hspace{\hspaceColsPC}
        \includegraphics[height=\heightPC, width=\widthPC]{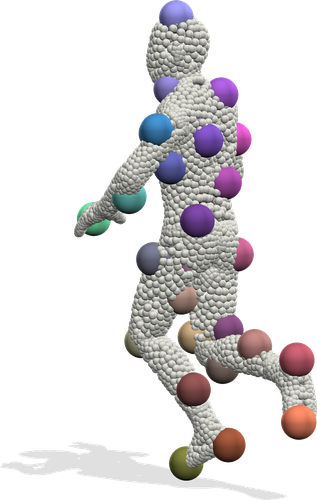}&
        \hspace{\hspaceColsPC}
        \includegraphics[height=\heightPC, width=\widthPC]{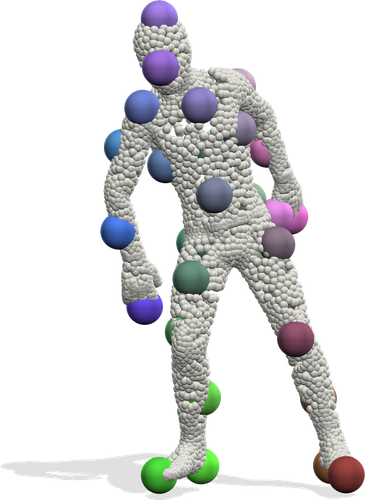}&
        \hspace{\hspaceColsPC}
        \includegraphics[height=\heightPC, width=\widthPC]{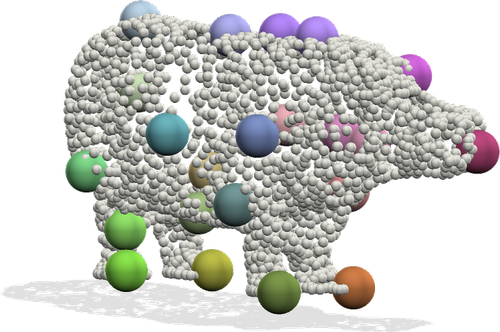}&
    \end{tabular}}
         \hspace{-0.6cm}
         \rotatebox{90}{\textcolor{gray}{\hspace{-1.9cm}\rule{4cm}{0.5pt}}} &
         \hspace{-0.1cm}
         \makecell{
            \begin{tabular}{cc}
                    \hspace{-0.7cm}
                    \includegraphics[width=0.24\columnwidth, height=2cm]{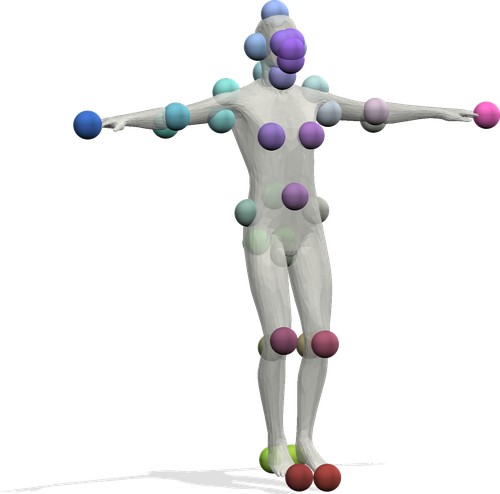}& 
                    \hspace{-0.8cm}
                    \includegraphics[width=0.24\columnwidth, height=2cm]{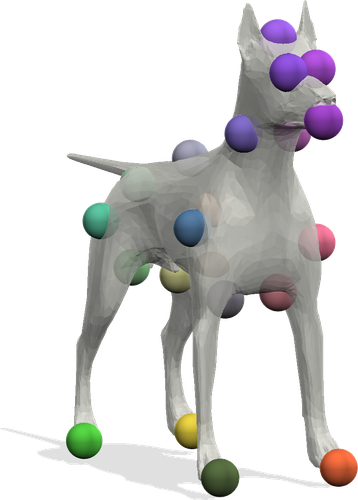} \\
                    \hspace{-1cm}
                    \includegraphics[width=0.24\columnwidth, height=2cm]{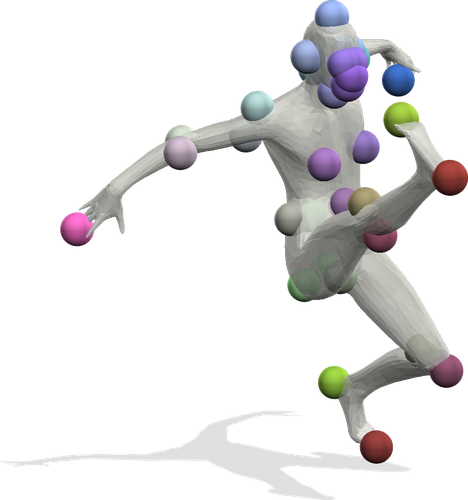}& 
                    \hspace{-0.8cm}
                    \includegraphics[width=0.24\columnwidth, height=2cm]{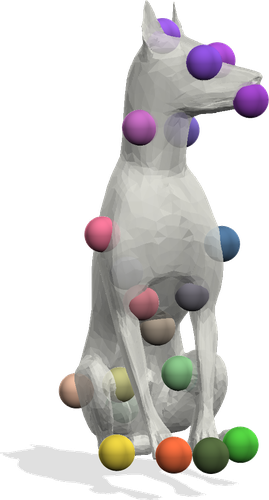}
            \end{tabular}
         }
    \end{tabular}
    \caption{(Left) \textbf{Mesh to point cloud} matching. Shapes in the top row are matched to respective shapes in the bottom row. Our method is able to tackle this challenging scenario. (Right) \textbf{failure modes}. For few shapes our orientation term does not help to disentangle intrinsic symmetry. }
    \label{fig:point-cloud-failures}
\end{figure}

\vspace{-0.5mm}
\section{Discussion \& Limitations} Despite the state-of-the-art performance, SIGMA has also limitations. 
We show some failure cases due to inconsistent orientation maps in Fig.~\ref{fig:point-cloud-failures}. 
As proof-of-concept, we showed in Fig.~\ref{fig:teaser}b that SIGMA can also work for partial shapes. However, its performance is far from perfect under this challenging scenario and often cannot find the global optimum within an adequate time limit. 
This is due to the fact that the equality constraints on $\mP$ become the inequalities $\mP^\top\vone \leq \vone,~\mP\vone \leq \vone$ (cf. Eqn.~\eqref{eq:optimization-problem}) under partiality, and this  increases the size of the  solution search space, resulting in longer runtime.
We can handle mesh to point cloud matching quite well, as shown in Fig.~\ref{fig:point-cloud-failures}.
Here, we replace the geodesic distance with Euclidean distances for the pruning strategy. 
Another challenge are matchings between pairs with topological changes. Here, SIGMA struggles since the mesh of one shape could not well-explain the deformation into the other shape anymore. We leave these challenging problem settings for future  exploration.

\vspace{-0.5mm}
\section{Conclusion} \label{sec:conclusion}
We presented SIGMA, an initialisation-free MIP formulation for sparse shape matching problems, for which we demonstrate that it can be solved to global optimality for many instances.
Our  method is provably invariant to global scaling and rigid transformations of the input shapes, eliminating the need for an extrinsic shape (pre-)alignment required by many shape matching methods. 
Furthermore, we introduced  the projected Laplace-Beltrami operator PLBO, which combines both intrinsic and extrinsic geometric information and remains invariant under $\rmE{3}$ group actions -- we believe it may also be useful for a broad range of other geometry processing tasks. In our experiments, we demonstrated that our  method outperforms competitive baselines in terms of matching quality, optimality and scalability, across a spectrum of challenging 3D shape benchmarks.

\blfootnote{\textbf{Acknowledgement}. This work were supported by the ERC Advanced Grant SIMULACRON, the CRC "Discretization in Geometry and Dynamics", the EPSRC Programme Grant VisualAI EP/T028572/1, the Ministry of Culture and Science NRW and MCML.
}
{\small
\bibliographystyle{ieee_fullname}
\bibliography{sigma_references}
}

\appendix
\clearpage
\begin{center}
\textbf{\large Supplementary Material}
\end{center}

\section{Proofs}

In the following we provide proofs for all lemmata from the main paper.
\subsection{Proof of Lemma 1} \label{subsec:proofs-lemma1}

\paragraph{Invariance of the PLBO}
We demonstrate that, despite using the vertex coordinates $\mX$ explicitly, the operator $\DeltaXproj$ is agnostic to the extrinsic orientation of the input pose $\cX$. Specifically, it is invariant under arbitrary rigid body transformations from the Euclidean group $\rmE{3}$.

\begin{lemma}\label{lemma:invariance-rigid-body-supp}
Let $\Delta(\mX):=\DeltaXproj\in\bbR^{\absmX\times \absmX}$ be the projected Laplace-Beltrami operator for the vertices $\mX$, defined in~Eqn.~\eqref{eq:proj-lbo}. For any rigid body transformation 
\begin{equation}
    \begin{pmatrix}
    \mR & \vt \\
    \vzero & 1
    \end{pmatrix}\in\rmE{3},~~\text{with}~~\mR\in \rmO{3},\vt\in\bbR^3,
\end{equation}
it holds that $\Delta(\mX)=\Delta(\mX\mR^\top+\vone\vt^\top)$.
\end{lemma}

\begin{proof}
We first simplify the rigidly transformed vertices $\mX\mR^\top+\vone\vt^\top$ in homogeneous coordinates 
\begin{equation}
\begin{pmatrix}\mX\mR^\top+\vone\vt^\top&\vone\end{pmatrix}=\begin{pmatrix}\mX&\vone\end{pmatrix}\begin{pmatrix}\mR^\top&\mathbf{0}\\\vt^\top&1\end{pmatrix}=\mXt\mQ^\top
\end{equation}
where $\mXt:=\begin{pmatrix}\mX&\vone\end{pmatrix}\in\bbR^{\absmX\times 4}$ is the neutral input pose and $\mQ:=\begin{pmatrix}
    \mR & \vt \\
    \vzero & 1
    \end{pmatrix}$. 
We then directly obtain the invariance of the projection matrix defined in Eqn.~\eqref{eq:proj-matrix}
\begin{align}\nonumber
    \mPi_{\mX\mR^\top+\vone\vt^\top}=&~\mathbf{I}-\mXt\mQ^\top(\mQ\mXt^\top\mXt\mQ^\top)^{-1}\mQ\mathbf{{X}}^\top\\=&~\mathbf{I}-\mXt(\mXt^\top\mXt)^{-1}\mXt^\top=\mPi_{\mX}.
    \label{eq:proof-projection-invariance}
\end{align}
The last equality is valid for any $\mQ$, because $\rmE{3}\subset\mathrm{GL}(4)$.
Since the Laplacian matrix $\DeltaXstiff$ is, by construction, invariant under rigid-body transformations, inserting Eqn.~\eqref{eq:proof-projection-invariance} into Eqn.~\eqref{eq:proj-lbo} directly yields the desired equality.
\end{proof}

\subsection{Proof of Lemma 2} \label{subsec:proofs-lemma2}
\begin{lemma}\label{lemma:invariance-scale-rigid-supp}
    Let $\bigl(\mP,\mXh,\mYh\bigr)$ be a global optimiser of Eqn.~\eqref{eq:optimization-problem}. 
    \begin{enumerate}
        \item[(a)] Let $\cX':=\bigl(s\mX,\mFX,\cI\bigr)$ be a rescaled input shape $\cX$, where a scalar factor $s>0$ is applied to the vertex coordinates.
        Then $\bigl(\mP',\mXh',\mYh'\bigr):=\bigl(\mP,\mXh,s\mYh\bigr)$ is a global optimiser of Eqn.~\eqref{eq:optimization-problem} between $\cX'$ and $\cY$.
        \item[(b)] Let $\cX'':=\bigl(\mX \mR^\top + \vone\vt^\top ,\mFX,\cI\bigr)$ be a rigidly transformed version of $\cX$ with $\mR\in \rmSO{3},\vt\in\bbR^3$.
    Then $\bigl(\mP'',\mXh'',\mYh''\bigr):=\bigl(\mP,\mXh,\mYh\mR^\top+\vone\vt^\top\bigr)$ is a global optimiser of Eqn.~\eqref{eq:optimization-problem} between $\cX''$ and $\cY$.
    \end{enumerate}
\end{lemma}

\paragraph{Lemma 2a -- Invariance of Global Scaling}
\begin{proof}
Rescaling the shape $\cX\to \cX'$ affects the shape diameter in the same manner $\dcX'=s~\dcX$ with $s>0$.
On the other hand, the orientation features $\vh^{\bullet}$ are fully scale invariant, since we leverage scale-invariant scalar input fields and the outer normals are normalised to unit length. Likewise, the Laplacian stiffness matrix $\DeltaXstiff$ is unaffected, and the projection $\mPiX$ defined in Eqn.~\eqref{eq:proj-matrix} of $\cX'$ becomes:

\begin{align}
    \label{eq:plbo-scale-invariant}
    \mPiXp &~= \mI - \mXt \mS \bigl( (\mXt \mS) ^\top \mXt \mS)^{-1} (\mXt \mS \bigr)^\top \nonumber \\
    &~= \mI - \mXt (\mXt^\top \mXt)^{-1} \mXt^\top = \mPiX,
\end{align}

where the homogeneous coordinates $\mXt$ are rescaled with the diagonal matrix $\mS:=\mathrm{diag}(s,s,s,1)\in\bbR^{4\times 4}$.
Hence, the projected LBO $\DeltaXproj$ from Eqn.~\eqref{eq:proj-lbo} scale-invariant.

Inserting these scale shift identities yields the following optimisation problem Eqn.~\eqref{eq:optimization-problem-scaled}:

\begin{align}\label{eq:optimization-problem-scaled}
    \min_{\mP', \mXh', \mYh'}~~&~~\frac{1}{n\dcY}\bigl\|\mXhI'-\mP'\mYJ\bigr\|_F+\frac{1}{n\dcX}\bigl\|\frac{1}{s}\mYhJ'-\mP'^\top\mXI\bigr\|_F
    \nonumber\\&~~+\frac{\lambda_\mathrm{def}}{\absmX\dcY}\bigl\|\DeltaXproj\mXh'\bigr\|_F+\frac{\lambda_\mathrm{def}}{\absmY\dcX}\bigl\|\frac{1}{s}\DeltaYproj\mYh'\bigr\|_F \nonumber\\
    &~~+\frac{\lambda_\mathrm{ori}}{n}\bigl\| \vhX-\mP'\vhY \bigr\|_F,\\
    \text{s.t.}~~&~~\mP'\in\{0,1\}^{n\times n},~\mP'^\top\vone_n=\vone_n,~\mP'\vone_n=\vone_n,\nonumber
\end{align}
Substituting $\mP'\to \mP$, $\mXh'\to \mXh$ and $\frac{1}{s}\mYh'\to \mYh$ then results in exactly the optimisation problem from Eqn.~\eqref{eq:optimization-problem} with the original inputs $\cX$ and $\cY$. Inserting the global optimiser from the original, unscaled problem $\bigl(\mP,\mXh,\mYh\bigr)$ directly results in the global optimiser $\bigl(\mP',\mXh',\mYh'\bigr)=\bigl(\mP,\mXh,s\mYh\bigr)$ of the scaled problem. 
\end{proof}

\paragraph{Lemma 2b -- Invariance of Rigid Transformations}

\begin{proof} Applying a rigid transformation $\cX \to \cX''$ to the shape $\mX \to \mX \mR^\top + \vone  \vt^\top $ leads to $\mXI \to \mXI \mR^\top + \vone  \vt^\top $ directly, where $\mR \in \rmSO{3}$ and $\vt \in \bbR^3$. On the other hand, $\DeltaXproj$ is invariant to rigid transformation (cf. \Cref{lemma:invariance-rigid-body}), so the term $\DeltaXproj$ stays unaffected. Furthermore, the orientation-aware features $\vh^{\bullet}$ are rigid transformation invariant, as discussed in Sec.~\ref{subsec:sparse-matching} of the main paper.

Inserting these rigid transformation identities yields the following optimisation problem Eqn.~\eqref{eq:optimization-problem-rigid}:

\begin{align}\label{eq:optimization-problem-rigid}
    \min_{\mP'', \mXh'', \mYh''}~~&~~\frac{1}{n\dcY}\bigl\|\mXhI''-\mP''\mYJ\bigr\|_F \nonumber \\
    &~~+\frac{1}{n\dcX}\bigl\|\mYhJ''-\mP''^\top (\mXI \mR^\top + \vone  \vt^\top) \bigr\|_F
    \nonumber\\&~~+\frac{\lambda_\mathrm{deform}}{\absmX\dcY}\bigl\| \DeltaXproj\mXh''\bigr\|_F+\frac{\lambda_\mathrm{deform}}{\absmY\dcX}\bigl\| \DeltaYproj\mYh''\bigr\|_F \nonumber\\
    &~~+\frac{\lambda_\mathrm{orient}}{n}\bigl\| \vhX-\mP''\vhY \bigr\|_F,\\
    \text{s.t.}~~&~~\mP''\in\{0,1\}^{n\times n},~\mP''^\top\vone_n=\vone_n,~\mP''\vone_n=\vone_n,\nonumber
\end{align}

Substituting $\mP''\to \mP$, $\mXh''\to \mXh$ and $\mYh'' \to \mYh \mR^\top + \vone \vt^\top $, we have:

\begin{align}\label{eq:optimization-problem-rigid1}
    \min_{\mP, \mXh, \mYh}~~&~~\frac{1}{n\dcY}\bigl\|\mXhI-\mP\mYJ\bigr\|_F \nonumber \\
    &~~+\frac{1}{n\dcX}\bigl\| (\mYhJ \mR^\top + \vone \vt^\top) - \mP^\top (\mXI \mR^\top + \vone  \vt^\top) \bigr\|_F
    \nonumber\\&~~+\frac{\lambda_\mathrm{deform}}{\absmX\dcY}\bigl\| \DeltaXproj\mXh\bigr\|_F \nonumber \\
    &~~+\frac{\lambda_\mathrm{deform}}{\absmY\dcX}\bigl\| \DeltaYproj(\mYh \mR^\top + \vone \vt^\top)\bigr\|_F \nonumber\\
    &~~+\frac{\lambda_\mathrm{orient}}{n}\bigl\| \vhX-\mP\vhY \bigr\|_F,\\
    \text{s.t.}~~&~~\mP\in\{0,1\}^{n\times n},~\mP^\top\vone_n=\vone_n,~\mP\vone_n=\vone_n,\nonumber
\end{align}

The first, third and fifth term in Eqn.~\eqref{eq:optimization-problem-rigid1} are same as the corresponding terms in the original problem of Eqn.~\eqref{eq:optimization-problem}, hence the only critical terms are the second and forth. The second can be rewritten as follows:

\begin{align}
    &~~\bigl\|\mYh \mR^\top + \vone \vt^\top -\mP^\top (\mXI \mR^\top + \vone \vt^\top) \bigr\|_F \nonumber \\
    =&~~\bigl\| (\mYh - \mP^\top  \mXI ) \mR^\top \bigr\|_F \nonumber \\
    =&~~\bigl\| \mYh - \mP^\top  \mXI \bigr\|_F 
\end{align}

The first equality holds because $\mP^\top\vone=\vone$ holds by construction. The second equality follows from the fact the $\rmSO{3}$ elements do not affect the Frobenius norm.

For the forth term we have the following equalities:

\begin{align}
    &~~\bigl\| \DeltaYproj(\mYh \mR^\top + \vone \vt^\top)\bigr\|_F \nonumber \\
    =&~~\bigl\| \DeltaYproj\mYh \mR^\top + \DeltaYproj \vone \vt^\top\bigr\|_F \nonumber \\
    =&~~\bigl\| \DeltaYproj\mYh \mR^\top \bigr\|_F \nonumber \\
    =&~~\bigl\| \DeltaYproj\mYh \bigr\|_F
\end{align}

where we utilise the linearity of the PLBO, and the fact that the constant vector $\vone \vt^\top$ lives in its nullspace (the first and second equality). The third equality holds because rotations do not affect the Frobenius norm.

Hence, $\bigl( \mP'', \mXh'', \mYh''\bigr)= \bigl( \mP,\mXh, \mYh \mR^\top + \vone \vt^\top \bigr)$ is the global optimiser of the rigidly transformed problem. 
\end{proof}

\section{Implementation Details}

\begin{figure}[t]
    \centering
    \begin{tabular}{cc}
        \includegraphics[width=3.5cm]{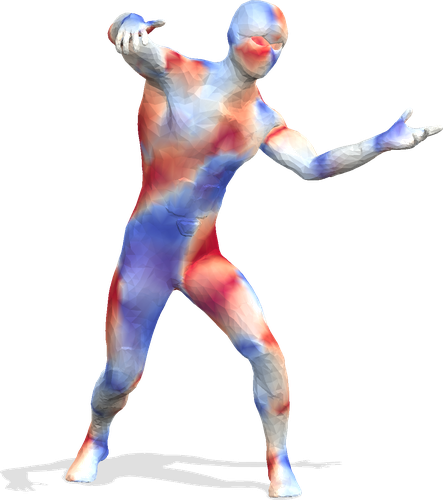}&
        \includegraphics[width=4cm]{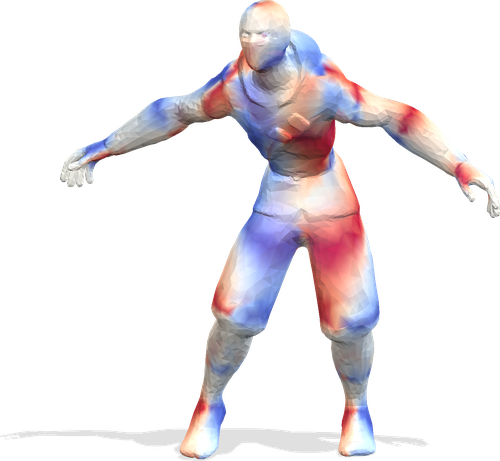}\\
        \includegraphics[width=3.5cm]{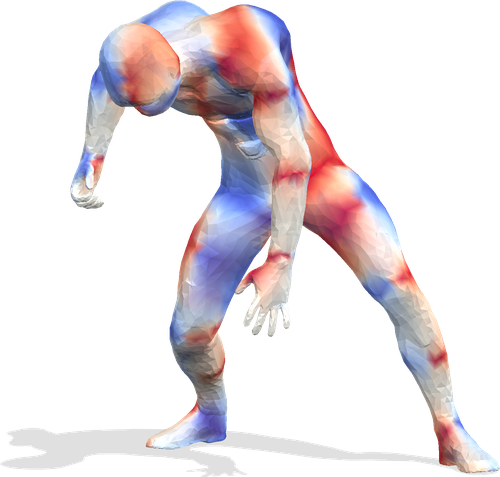}&
        \includegraphics[width=3.7cm]{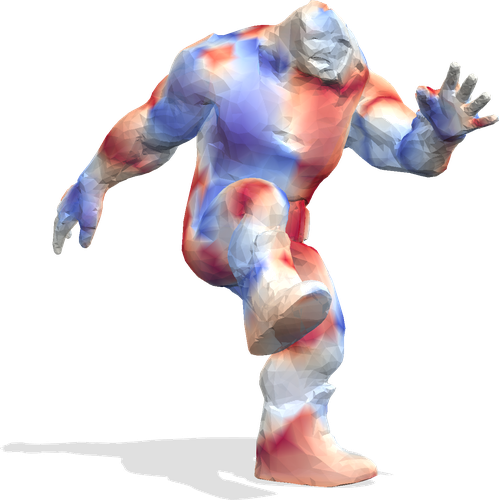}\\
        \includegraphics[width=3.7cm]{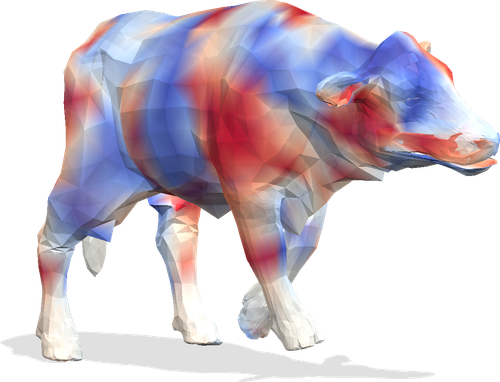}&
        \includegraphics[width=2.8cm]{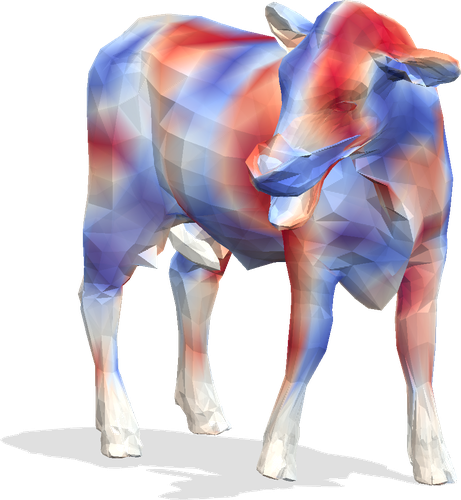}\\
        \includegraphics[width=3.5cm]{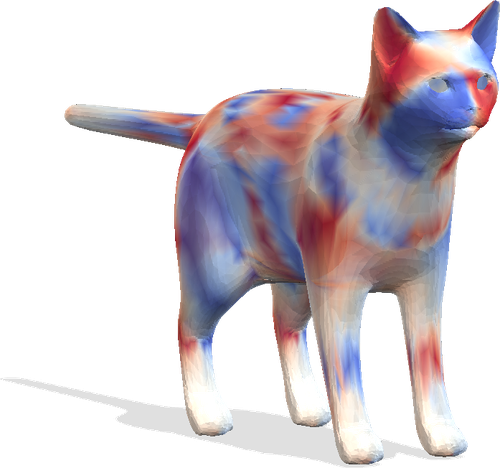}&
        \includegraphics[width=2.5cm]{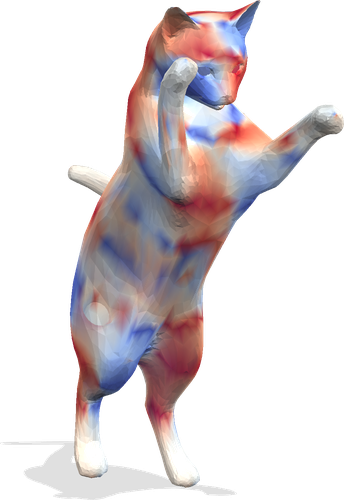}
    \end{tabular}
    \caption{Visualisation of our \textbf{orientation-aware feature}. Red and blue indicate the respective orientation ${-1, 1}$ while white indicates values in-between.
    }
    \label{fig:feat_ori}
\end{figure}

\subsection{Shape Reconstruction}
\label{secsupp:shape_rec}
In Fig.~\ref{fig:lbovsproj} examples of the reconstructed shapes using the area-normalised LBO, \ie the stiffness matrix component $\DeltaXstiff$ of the Laplacian, and PLBO are illustrated for qualitative comparison. It demonstrates that the PLBO leads to more realistic results whereas the LBO leads to over-smoothed reconstructions. For obtain these results, we fixed the correspondences in the $\cErec$ term of our final objective defined in Eqn.~\eqref{eq:optimization-problem} and run our full implementation to estimate $\mXh$ and $\mYh$. This results in fast optimisation since the unknowns are continuous variables living in $\bbR^3$ and helps to show the quality of the reconstruction in very high resolution ($10$k faces) with clean correspondences. Since the reconstructed shapes are only a by-product of our method, we use lower resolution meshes ($500$ faces) in all experiments for the sake of faster optimisation, but a similar reconstruction behaviour of PLBO and LBO can be observed in this case as well.

The idea of projecting LBO has been explored in \cite{nasikun2018fastspectrum}, with a focus on \emph{preserving} a certain subspace of the LBO (approximating the low frequency end of the spectrum), for the task of fast spectral decomposition. In contrary, we aim to \emph{exclude} the subspace of shape coordinates by the adding it to the kernel (of PLBO) for geometry reconstruction. Hence both approaches can be seen complementary.

\subsection{Details on the Orientation-Aware Feature}

Inspired by \cite{ren2018continuous}, the orientation-aware feature defined in Eqn.~\eqref{eq:orientation-feature} requires two scalar-valued features $\vfX$ and $\vgX$ besides the unit outer normal $\vnX$ as input which encode the information of shape orientation. In theory, we can use any pair of scalar-valued features to construct an orientation-aware feature map. However, there are several aspects one needs to take into consideration. First, the scalar-values features should be scale invariant. Second, their gradient fields should not be close to parallel, since the cross-product of two parallel vectors will vanish. In practice, we choose the $1$st and $70$-th frequency of the wave kernel signature \cite{aubry2011wks} as the $\vfX$ and $\vgX$ due to their dissimilar frequencies which effectively avoids near-parallel gradient fields. Empirically we found it works well for our purpose. See Fig.~\ref{fig:feat_ori} for an illustration. Note that although the orientation-aware features can be easily computed densely for every vertex, only the ones for the sparse keypoints are utilised in Eqn.~\eqref{eq:term-ori-aware}. As shown in Fig.~\ref{fig:feat_ori}, the feature map is noisy, especially under the presence of non-isometric deformations. Hence, it helps to disambiguate the intrinsic shape symmetry, however, it does not fully exclude symmetrically flipped matchings and does not lead to fine-grain feature alignment.

\begin{figure}
    \begin{tabular}{cc}
        \hspace{-1cm}
        \newcommand{\pckLineWidth}{3pt}
\newcommand{\plotWidth}{\columnwidth}
\newcommand{\plotHeight}{0.75\columnwidth}
\newcommand{\pckTitle}{TOSCA}
\definecolor{cPLOT0}{RGB}{28,213,227}
\definecolor{cPLOT1}{RGB}{80,150,80}
\definecolor{cPLOT2}{RGB}{90,130,213}
\definecolor{cPLOT3}{RGB}{247,179,43}
\definecolor{cPLOT5}{RGB}{242,64,0}

\pgfplotsset{%
    label style = {font=\large},
    tick label style = {font=\large},
    title style =  {font=\LARGE},
    legend style={  fill= gray!10,
                    fill opacity=0.6, 
                    font=\large,
                    draw=gray!20, %
                    text opacity=1}
}
\begin{tikzpicture}[scale=0.5, transform shape]
	\begin{axis}[
		width=\plotWidth,
		height=\plotHeight,
		grid=major,
		title=\pckTitle,
		legend style={
			at={(0.97,0.03)},
			anchor=south east,
			legend columns=1},
		legend cell align={left},
		ylabel={{\Large$\%$ Correct Matchings}},
        xlabel={{\Large$\%$ Geodesic Error}},
		xmin=0,
        xmax=1,
        ylabel near ticks,
        xtick={0, 0.25, 0.5, 0.75, 1},
        ymin=0.5,
        ymax=1.03,
        ytick={0, 0.20, 0.40, 0.60, 0.7, 0.80, 0.9, 1.00},
        yticklabels = {0, 20, 40, 60, 70, 80, 90, 100},
	]

    \addplot [color=cPLOT3, smooth, line width=\pckLineWidth]
    table[row sep=crcr]{%
0	0.831722217362144\\
0.0344827586206897	0.855876806734124\\
0.0689655172413793	0.895471865006464\\
0.103448275862069	0.918214601619378\\
0.137931034482759	0.933608899775464\\
0.172413793103448	0.942271989424469\\
0.206896551724138	0.951046860875397\\
0.241379310344828	0.956077041961916\\
0.275862068965517	0.958480350703253\\
0.310344827586207	0.959177771945684\\
0.344827586206897	0.960239699063949\\
0.379310344827586	0.965605225556236\\
0.413793103448276	0.971082533850445\\
0.448275862068966	0.973344900319793\\
0.482758620689655	0.974518609239981\\
0.517241379310345	0.975130979111383\\
0.551724137931034	0.975913451724841\\
0.586206896551724	0.976219636660543\\
0.620689655172414	0.977002109274001\\
0.655172413793103	0.979655712050078\\
0.689655172413793	0.981220657276995\\
0.724137931034483	0.982003129890454\\
0.758620689655172	0.984741784037559\\
0.793103448275862	0.988262910798122\\
0.827586206896552	0.991001564945227\\
0.862068965517241	0.992175273865415\\
0.896551724137931	0.997652582159624\\
0.931034482758621	1\\
0.96551724137931	1\\
1	1\\
    };
    \addlegendentry{\textcolor{black}{{Mina}~\cite{bernard2020mina}: 0.962}}
    \addplot [color=cPLOT5, smooth, densely dotted, line width=\pckLineWidth]
    table[row sep=crcr]{%
0	0.81345561290448\\
0.0344827586206897	0.839272349070267\\
0.0689655172413793	0.87388339700036\\
0.103448275862069	0.907104462523936\\
0.137931034482759	0.932077975096958\\
0.172413793103448	0.953107534093449\\
0.206896551724138	0.977738411143187\\
0.241379310344828	0.986595903925971\\
0.275862068965517	0.994138745516578\\
0.310344827586207	0.997433878634124\\
0.344827586206897	0.99810456944566\\
0.379310344827586	0.99810456944566\\
0.413793103448276	0.998775260257195\\
0.448275862068966	0.998775260257195\\
0.482758620689655	0.998775260257195\\
0.517241379310345	0.998775260257195\\
0.551724137931034	1\\
0.96551724137931	1\\
1	1\\
    };
    \addlegendentry{\textcolor{black}{Ours w/o Ori.: 0.976}}
    \addplot [color=cPLOT5, smooth, line width=\pckLineWidth]
    table[row sep=crcr]{%
0	0.890417383528222\\
0.0344827586206897	0.900533636602222\\
0.0689655172413793	0.917548770886187\\
0.103448275862069	0.937049835243344\\
0.137931034482759	0.952864530176226\\
0.172413793103448	0.963981959389185\\
0.206896551724138	0.980904752184605\\
0.241379310344828	0.986464681810671\\
0.275862068965517	0.994255387396845\\
0.310344827586207	0.996821508762721\\
0.344827586206897	0.997798384509958\\
0.379310344827586	0.997798384509958\\
0.413793103448276	0.998775260257195\\
0.448275862068966	0.998775260257195\\
0.482758620689655	0.998775260257195\\
0.517241379310345	0.998775260257195\\
0.551724137931034	1\\
1	1\\
    };
    \addlegendentry{\textcolor{black}{Ours: \textbf{0.984}}}
        
	\end{axis}
\end{tikzpicture}&
        \hspace{-1cm}
        \newcommand{\gapLineWidth}{3pt}
\definecolor{cPLOT0}{RGB}{28,213,227}
\definecolor{cPLOT1}{RGB}{80,150,80}
\definecolor{cPLOT2}{RGB}{90,130,213}
\definecolor{cPLOT3}{RGB}{247,179,43}
\definecolor{cPLOT5}{RGB}{242,64,0}

\pgfplotsset{%
    label style = {font=\large},
    tick label style = {font=\large},
    title style =  {font=\LARGE},
    legend style={  fill= gray!10,
                    fill opacity=0.6, 
                    font=\large,
                    draw=gray!20, %
                    text opacity=1}
}
\begin{tikzpicture}[scale=0.5, transform shape]
	\begin{axis}[
		width=\columnwidth,
		height=0.75\columnwidth,
		grid=major,
		title=Optimality vs. Runtime,
		legend style={
			at={(0.97,0.03)},
			anchor=south east,
			legend columns=1},
		legend cell align={left},
		ylabel={{\Large$\%$ Pairs w/ Closed Gap}},
        xlabel={{\Large$\%$ Runtime (s)}},
		xmin=0,
        xmax=3650,
        ylabel near ticks,
        xtick={0, 1000, 2000, 3000, 4000},
		ymin=-5,
        ymax=103,
        ytick={0, 20, 40, 60, 80, 100},
	]
 
    \addplot [color=cPLOT3, smooth, line width=\gapLineWidth]
    table[row sep=crcr]{%
0 0\\
124.1379 7.0423\\
248.2759 25.3521\\
372.4138 29.5775\\
496.5517 32.3944\\
620.6897 39.4366\\
744.8276 39.4366\\
868.9655 40.8451\\
993.1034 46.4789\\
1117.2414 49.2958\\
1241.3793 50.7042\\
1365.5172 53.5211\\
1489.6552 54.9296\\
1613.7931 56.338\\
1737.931 57.7465\\
1862.069 59.1549\\
1986.2069 63.3803\\
2110.3448 63.3803\\
2234.4828 66.1972\\
2358.6207 66.1972\\
2482.7586 66.1972\\
2606.8966 66.1972\\
2731.0345 66.1972\\
2855.1724 69.0141\\
2979.3103 73.2394\\
3103.4483 73.2394\\
3227.5862 73.2394\\
3351.7241 73.2394\\
3475.8621 76.0563\\
3600 77.4648\\
    };
    \addlegendentry{\textcolor{black}{Mina}~\cite{bernard2020mina}}

    \addplot [color=cPLOT5, smooth, tension=0.1, line width=\gapLineWidth]
    table[row sep=crcr]{%
0 0\\
124.1379 61.9718\\
248.2759 84.507\\
372.4138 92.9577\\
496.5517 95.7746\\
620.6897 100\\
3600 100\\
    };
    \addlegendentry{\textcolor{black}{Ours}}
        
	\end{axis}
\end{tikzpicture}
    \end{tabular}
    \caption{Quantitative results with $k=5$ for MINA~\cite{bernard2020mina} and ours. Our method produces global optimal results for \textbf{all} instances for this (smaller) optimisation problem within $15$min.}
    \label{fig:pck-tosca-numlaps5}
\end{figure}
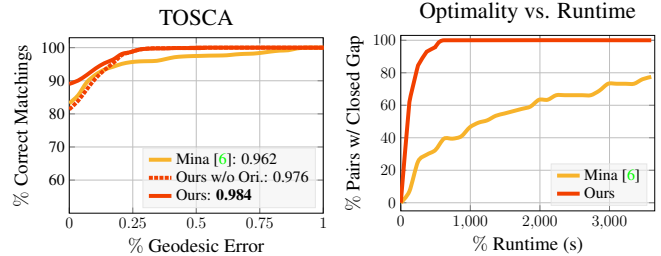

\section{Further Experiments}

\subsection{Experiments with Default Settings in MINA~\cite{bernard2020mina}}

The results shown in the main paper are obtained under a more conservative search space pruning than was proposed in the original paper, namely $k=11$, where the original MINA defaults to $k=5$. Below we show additional experimental results using $k=5$ and argue that in general less pruning is more favourable. 

As shown in Fig.~\ref{fig:pck-tosca-numlaps5}, MINA achieves accurate correspondences and certifies $80\%$ globally optimal pairs within $1$h in its original setting with a solution search space restricted to only allowing $k=5$ matching candidates per keypoint. In this setting, SIGMA still outperforms MINA and is able to find the global optima of all $71$ TOSCA \cite{Bronstein:2008:NGN:1462123} matching pairs with better accuracy within $15$ minutes. Note that the matching accuracy of SIGMA does not improve compared to the case of $k=11$ (cf. Fig.~\ref{fig:pck_tosca}), this is because the more aggressive pruning inevitably excludes some correct matchings which effects the matching performance negatively. Therefore, a more conservative pruning with higher $k$ is preferable in practice. Removing the pruning completely leads to an extremely enlarged search space and, thus, also higher runtime.

\subsection{Robustness to Keypoints}
Compared to MINA, our SIGMA is less sensitive to the exact position of keypoints.
Our experiments (Fig.~\ref{fig:pck_tosca}, \ref{fig:q-comparison}) on SHREC20 non-isometric shows the robustness of SIGMA, since the keypoints were hand-picked by the dataset author and an absolute matching does not exist.
Additionally we consider the challenging scenario of sampling the keypoints for each shape independently using farthest point sampling (FPS). Here, there are no guarantees that meaningful correspondences exist.
Quantitative and qualitative results on SMAL can be seen in Fig.~\ref{fig:pck_smal_fps}, \ref{fig:qual_smal_fps}. 
Compared to other keypoint-dependent approaches (MINA and PMSDP), our method is the least sensitive to keypoint positions but a certain drop in performance is expected.

\begin{figure}
    \centering
    \newcommand{\pckLineWidth}{3pt}
\newcommand{\plotWidth}{1.5\columnwidth}
\newcommand{\plotHeight}{1\columnwidth}
\newcommand{\pckTitle}{PCK of SMAL w/ fps Sampling}
\definecolor{cPLOT0}{RGB}{28,213,227}
\definecolor{cPLOT1}{RGB}{80,150,80}
\definecolor{cPLOT2}{RGB}{90,130,213}
\definecolor{cPLOT3}{RGB}{247,179,43}
\definecolor{cPLOT5}{RGB}{242,64,0}

\pgfplotsset{%
	label style = {font=\large},
	tick label style = {font=\large},
	title style =  {font=\large,
                        yshift=-0.5em},
	legend style={  fill= gray!10,
		fill opacity=0.6, 
		font=\large,
		draw=gray!20, %
		text opacity=1}
}
\begin{tikzpicture}[scale=0.5, transform shape]
	\begin{axis}[
		width=\plotWidth,
		height=\plotHeight,
		grid=major,
		legend style={
			at={(0.97,0.03)},
			anchor=south east,
			legend columns=1},
		legend cell align={left},
		ylabel={{\LARGE$\%$ Correct Matchings}},
        xlabel={\LARGE$\%$ Geodesic Error},
		xmin=0,
		xmax=1,
		ylabel near ticks,
		xtick={0, 0.25, 0.5, 0.75, 1},
		ymin=0,
		ymax=1.03,
		ytick={0, 0.20, 0.40, 0.60, 0.80, 1.00},
		yticklabels={$0$, $20$, $40$, $60$, $80$, $100$},
		]
		\addplot [color=cPLOT2, smooth, line width=\pckLineWidth]
		table[row sep=crcr]{%
			0	0.00340301409820127\\
			0.0344827586206897	0.0224842002916869\\
			0.0689655172413793	0.0330578512396694\\
			0.103448275862069	0.0444214876033058\\
			0.137931034482759	0.0618619348565873\\
			0.172413793103448	0.0955274671852212\\
			0.206896551724138	0.140738940204181\\
			0.241379310344828	0.188684978123481\\
			0.275862068965517	0.245685464268352\\
			0.310344827586207	0.305542051531356\\
			0.344827586206897	0.362421001458435\\
			0.379310344827586	0.415714632960622\\
			0.413793103448276	0.469798249878464\\
			0.448275862068966	0.52722411278561\\
			0.482758620689655	0.586290714632961\\
			0.517241379310345	0.649003403014098\\
			0.551724137931034	0.701324744773943\\
			0.586206896551724	0.745563928050559\\
			0.620689655172414	0.787858531842489\\
			0.655172413793103	0.825413223140496\\
			0.689655172413793	0.859200291686923\\
			0.724137931034483	0.885695187165775\\
			0.758620689655172	0.909212445308702\\
			0.793103448275862	0.931696645600389\\
			0.827586206896552	0.949866310160428\\
			0.862068965517241	0.961959163830821\\
			0.896551724137931	0.974477394263491\\
			0.931034482758621	0.985719494409334\\
			0.96551724137931	0.993072435585805\\
			1	1\\
		};
		\addlegendentry{\textcolor{black}{PMSDP tuned: 0.57}}
		\addplot [color=cPLOT3, smooth, line width=\pckLineWidth]
		table[row sep=crcr]{%
			0	0.0423029345759219\\
			0.0344827586206897	0.23829300737488\\
			0.0689655172413793	0.33949492115513\\
			0.103448275862069	0.413964662154802\\
			0.137931034482759	0.489941877445778\\
			0.172413793103448	0.568744223930345\\
			0.206896551724138	0.630764026263428\\
			0.241379310344828	0.680754451890621\\
			0.275862068965517	0.71732849803632\\
			0.310344827586207	0.754051558860301\\
			0.344827586206897	0.789479385277505\\
			0.379310344827586	0.807767748207796\\
			0.413793103448276	0.817562979791565\\
			0.448275862068966	0.831576048404696\\
			0.482758620689655	0.850364934365155\\
			0.517241379310345	0.867641785661405\\
			0.551724137931034	0.892786939812761\\
			0.586206896551724	0.908863002661435\\
			0.620689655172414	0.921476914011889\\
			0.655172413793103	0.930821678530099\\
			0.689655172413793	0.943090708442669\\
			0.724137931034483	0.952606548420595\\
			0.758620689655172	0.965116696512643\\
			0.793103448275862	0.972588044522528\\
			0.827586206896552	0.98013465644167\\
			0.862068965517241	0.9833843578139\\
			0.896551724137931	0.987934246413092\\
			0.931034482758621	0.988532332537494\\
			0.96551724137931	0.994705326792647\\
			1	1\\
		};
		\addlegendentry{\textcolor{black}{Mina: 0.80}}
		\addplot [color=cPLOT5, smooth, densely dotted, line width=\pckLineWidth]
		table[row sep=crcr]{%
			0	0.0444214876033058\\
			0.0344827586206897	0.28099173553719\\
			0.0689655172413793	0.417355371900826\\
			0.103448275862069	0.511363636363636\\
			0.137931034482759	0.580578512396694\\
			0.172413793103448	0.671487603305785\\
			0.206896551724138	0.741735537190083\\
			0.241379310344828	0.788223140495868\\
			0.275862068965517	0.825413223140496\\
			0.310344827586207	0.847107438016529\\
			0.344827586206897	0.862603305785124\\
			0.379310344827586	0.870867768595041\\
			0.413793103448276	0.877066115702479\\
			0.448275862068966	0.885330578512396\\
			0.482758620689655	0.894628099173553\\
			0.517241379310345	0.912190082644628\\
			0.551724137931034	0.932851239669421\\
			0.586206896551724	0.940082644628099\\
			0.620689655172414	0.944214876033058\\
			0.655172413793103	0.962809917355372\\
			0.689655172413793	0.970041322314049\\
			0.724137931034483	0.976239669421488\\
			0.758620689655172	0.987603305785124\\
			0.793103448275862	0.989669421487603\\
			0.827586206896552	0.993801652892562\\
			0.862068965517241	0.994834710743802\\
			0.896551724137931	0.995867768595041\\
			0.931034482758621	0.996900826446281\\
			0.96551724137931	0.99896694214876\\
			1	1\\
		};
		\addlegendentry{\textcolor{black}{Ours w/o Ori.: \textbf{0.85}}}
		\addplot [color=cPLOT5, smooth, line width=\pckLineWidth]
		table[row sep=crcr]{%
			0	0.0516528925619835\\
			0.0344827586206897	0.273760330578512\\
			0.0689655172413793	0.402892561983471\\
			0.103448275862069	0.491735537190083\\
			0.137931034482759	0.56301652892562\\
			0.172413793103448	0.658057851239669\\
			0.206896551724138	0.722107438016529\\
			0.241379310344828	0.762396694214876\\
			0.275862068965517	0.794421487603306\\
			0.310344827586207	0.815082644628099\\
			0.344827586206897	0.833677685950413\\
			0.379310344827586	0.84400826446281\\
			0.413793103448276	0.852272727272727\\
			0.448275862068966	0.863636363636364\\
			0.482758620689655	0.878099173553719\\
			0.517241379310345	0.894628099173554\\
			0.551724137931034	0.913223140495868\\
			0.586206896551724	0.929752066115702\\
			0.620689655172414	0.93801652892562\\
			0.655172413793103	0.949380165289256\\
			0.689655172413793	0.958677685950413\\
			0.724137931034483	0.962809917355372\\
			0.758620689655172	0.977272727272727\\
			0.793103448275862	0.982438016528926\\
			0.827586206896552	0.985537190082645\\
			0.862068965517241	0.987603305785124\\
			0.896551724137931	0.990702479338843\\
			0.931034482758621	0.993801652892562\\
			0.96551724137931	0.997933884297521\\
			1	1\\
		};
		\addlegendentry{\textcolor{black}{Ours: {0.84}}}
	\end{axis}
\end{tikzpicture}
    \caption{PCK curves on SMAL \cite{Zuffi:CVPR:2017} with independent FPS Sampling.}
    \label{fig:pck_smal_fps}
\end{figure}
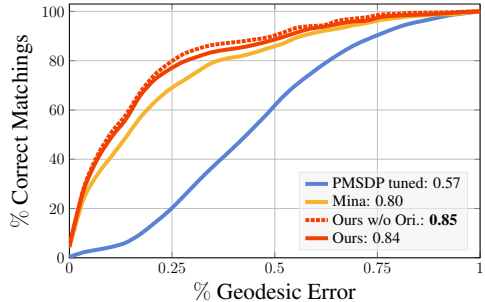

\begin{figure}[h!]
    \centering
    \begin{tabular}{cc}
        \hspace{-1cm}
         \newcommand{\scaleLinewWidth}{3pt}
\definecolor{cPLOT0}{RGB}{28,213,227}
\definecolor{cPLOT1}{RGB}{80,150,80}
\definecolor{cPLOT2}{RGB}{90,130,213}
\definecolor{cPLOT3}{RGB}{247,179,43}
\definecolor{cPLOT5}{RGB}{242,64,0}

\pgfplotsset{%
    label style = {font=\large},
    tick label style = {font=\large},
    title style =  {font=\Large},
    legend style={  fill= gray!10,
                    fill opacity=0.6, 
                    font=\large,
                    draw=gray!20, %
                    text opacity=1}
}
\begin{tikzpicture}[scale=0.5, transform shape]
	\begin{axis}[
		width=\columnwidth,
		height=0.75\columnwidth,
		grid=major,
        ylabel={{\Large Mean Geo. Err. (x100)}},
        xlabel={\resizebox{1cm}{!}{$\lambdarec$}},
        xmode=log,
        xmin=1e-8,
        xmax=1e6,
        ylabel near ticks,
        xtick={1e-8, 1e-4, 1e0, 1e4},
        ymin=3,
        ymax=14,
        ytick={2, 4, 6, 8, 10, 12, 14},
	]
    \addplot [color=cPLOT2, smooth, line width=\scaleLinewWidth]
    table[row sep=crcr]{%
1e-08	13.450787347693\\
1e-07	13.3455089712506\\
1e-06	10.0897186881428\\
1e-05	4.97000665063424\\
0.0001	6.13970581335309\\
0.001	7.83667435080871\\
0.01	7.3730484449518\\
0.1	6.08162310934488\\
0.2	5.99646191208002\\
0.3	5.28306550548388\\
0.4	5.86669751468636\\
0.5	5.85144961002569\\
0.6	5.6701390008938\\
0.7	5.64318302345769\\
0.8	5.39934592953973\\
0.9	5.39567345663631\\
1	5.55249957222243\\
10	6.51740863854626\\
100	6.42390444508719\\
1000	7.73035808174131\\
10000	7.64857916570787\\
100000	8.30702264501527\\
    };
        
	\end{axis}
\end{tikzpicture}%
         \hspace{-0.1cm}
         \definecolor{cPLOT0}{RGB}{28,213,227}
\definecolor{cPLOT1}{RGB}{80,150,80}
\definecolor{cPLOT2}{RGB}{90,130,213}
\definecolor{cPLOT3}{RGB}{247,179,43}
\definecolor{cPLOT5}{RGB}{242,64,0}

\pgfplotsset{%
    label style = {font=\large},
    tick label style = {font=\large},
    title style =  {font=\Large},
    legend style={  fill= gray!10,
                    fill opacity=0.6, 
                    font=\large,
                    draw=gray!20, %
                    text opacity=1}
}
\begin{tikzpicture}[scale=0.5, transform shape]
	\begin{axis}[
		width=\columnwidth,
		height=0.75\columnwidth,
		grid=major,
        xlabel={\resizebox{1cm}{!}{$\lambdaori'$}},
        xmode=log,
        xmin=1e-8,
        xmax=1e0,
        ylabel near ticks,
        xtick={1e-8, 1e-6, 1e-4, 1e-2, 1e0},
        ymin=3,
        ymax=14,
        ytick={2, 4, 6, 8, 10, 12, 14},
	]

 \addplot [color=cPLOT1, smooth, line width=\scaleLinewWidth]
    table[row sep=crcr]{%
1e-08	5.30286409764703\\
1e-07	5.31287289801323\\
1e-06	5.36020382491791\\
1e-05	5.23014612880729\\
0.0001	4.97072425396912\\
0.001	4.75244528307926\\
0.005	3.20845213788204\\
0.01	3.45024855095021\\
0.1	6.64317875484633\\
1	11.69934045965\\
10	12.8082280133184\\
100	13.6161056661484\\
    };
    \addlegendentry{\textcolor{black}{$\expnumber{3}{-1}$}}
    \addplot [color=cPLOT3, smooth, line width=\scaleLinewWidth]
    table[row sep=crcr]{%
1e-08	5.19890054712137\\
1e-07	5.40299117021706\\
1e-06	5.02241567725641\\
1e-05	5.15840706505248\\
0.0001	5.18794165012843\\
0.001	4.8607331316882\\
0.01	4.73164414541775\\
0.1	6.22503745898951\\
1	10.9134015414496\\
10	12.675278310894\\
100	13.610146712369\\
1000	13.9641398777638\\
10000	14.1061997953913\\
100000	13.8156056041571\\
    };
    \addlegendentry{\textcolor{black}{$\expnumber{9}{-1}$}}
    \addplot [color=cPLOT0, smooth, line width=\scaleLinewWidth]
    table[row sep=crcr]{%
1e-08	5.79194455330552\\
1e-07	5.21412565714623\\
1e-06	5.31553117120456\\
1e-05	10.1619749052211\\
0.0001	12.3500165492351\\
0.001	13.4999130454192\\
0.01	13.4808440477379\\
0.1	13.7895394534531\\
1	13.863597851476\\
10	13.8672782271962\\
100	13.8526645938296\\
    };
    \addlegendentry{\textcolor{black}{$\expnumber{1}{-5}$}}
	\end{axis}
\end{tikzpicture}%
    \end{tabular}
    \caption{ \textbf{Ablation study} on $\lambdarec~\&~\lambdaori'$. (Left) By setting $\lambdaori'=0$, the lowest mean geodesic errors are obtained at $\lambdarec~\in~\{\expnumber{1}{-5},~\expnumber{3}{-1},~\expnumber{9}{-1}\}$ (Right) The $\lambdarec$ is fixed to be $\{\expnumber{1}{-5},~\expnumber{3}{-1},~\expnumber{9}{-1}\}$ respectively, the lowest mean geodesic error is obtained at $\lambdaori'=\expnumber{5}{-3}$. Note that overall the (cyan) curve with $\lambdarec=\expnumber{1}{-5}$ lies above the other two curves, suggesting its suboptimal performance.}
    \label{fig:hyperparams-lambda}
\end{figure}
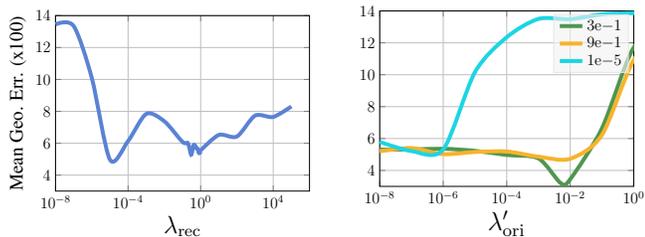

\begin{figure}
    \centering
    \begin{minipage}{0.9\columnwidth}
    \vspace{0.1cm}
         \begin{minipage}{0.8\columnwidth}
            \vspace{0cm}
             \begin{tabular}{cc}
                  \hspace{0cm}
                  \includegraphics[width=0.6\columnwidth]{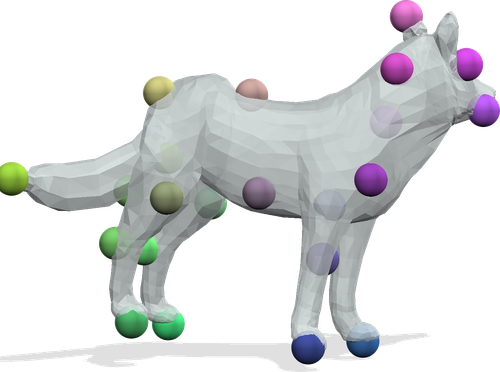}&
                  \hspace{0cm}
                  \includegraphics[width=0.6\columnwidth]{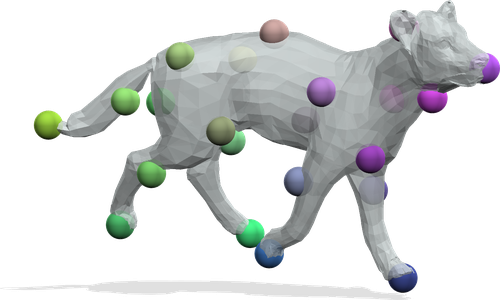}  \\
                  \includegraphics[width=0.6\columnwidth]{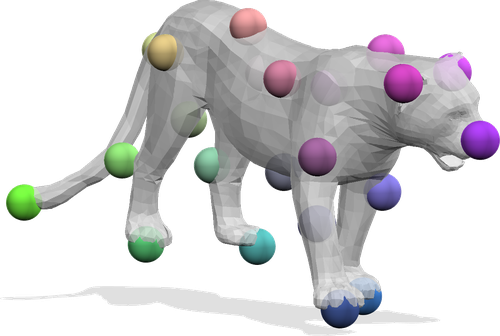}&
                  \hspace{-0.7cm}
                  \includegraphics[width=0.5\columnwidth]{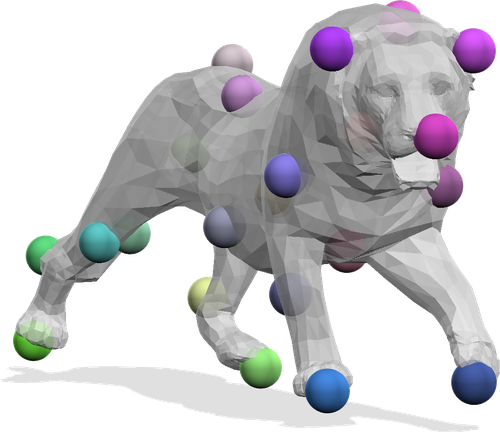}  \\
             \end{tabular}
         \end{minipage}
    \vspace{0.1cm}
    \end{minipage}
    \caption{Qualitative results of SIGMA on SMAL \cite{Zuffi:CVPR:2017} with independently FPS-sampled keyponits.}
    \label{fig:qual_smal_fps}
\end{figure}

\subsection{Ablation on $\lambdarec~\&~\lambdaori'$}

Our ablation is based on $20$ randomly sampled TOSCA shape pairs and an (equivalent) version of the original objective presented in Eq.~\eqref{eq:optimization-problem} for practical reasons.
\begin{equation}
\label{eq:optimisation-problem-lambda}
    \lambdarec \cErec(\mP, \mXh, \mYh) +\cEdef(\mXh, \mYh) + \lambdaori' \cEori(\mP)
\end{equation} 
We note it can be easily converted to the formulation in Eq.~\eqref{eq:optimization-problem} by setting $\lambdarec=\lambdadef^{-1}$ and $\lambdaori'=\lambdaori*\lambdarec$. Moreover, all the experiments (except otherwise mentioned) are conducted under the setting of $\lambdadef=5$ and $\lambdaori=\expnumber{2.5}{-2}$, which corresponds to $\lambdarec=\expnumber{2}{-1},~\lambdaori'=\expnumber{5}{-3}$.

As both the $\cErec$ and $\cEdef$ terms in the objective are dispensable (cf. Eq.~(\ref{eq:optimization-problem}, \ref{eq:optimisation-problem-lambda})), we first conduct the ablation on $\lambdarec$ alone, i.e. $\lambdaori'=0$. The search range covers from $\expnumber{1}{-8}$ to $\expnumber{1}{5}$ in logarithmic scale, with a refined search range between $\expnumber{1}{-1}$ and $\expnumber{1}{0}$. As shown in Fig.~\ref{fig:hyperparams-lambda} (left), the set of minimal mean geodesic error are achieved at $\{\expnumber{1}{-5},~\expnumber{3}{-1},~\expnumber{9}{-1}\}$. 

Consequently, we fix  $\lambdarec$ to be $\{\expnumber{1}{-5},~\expnumber{3}{-1},~\expnumber{9}{-1}\}$ respectively and fine tune $\lambdaori'$. The quantitative results in Fig.~\ref{fig:hyperparams-lambda} (right) suggest that the best accuracy is obtained at $\lambdaori'=\expnumber{5}{-3}$. 
In summary, the optimal setting (for the 20 randomly subsampled TOSCA pairs) is $\lambdarec=\expnumber{3}{-1},~\lambdaori'=\expnumber{5}{-3}$, which is close to the setting chosen in the main paper.

\end{document}